\documentclass{article}

\usepackage[preprint]{neurips_2026} 
\usepackage{amsmath,amssymb,amsfonts,amsthm}
\usepackage{booktabs}
\usepackage{multirow}
\usepackage{graphicx}
\usepackage{microtype}
\usepackage{xcolor}
\usepackage{tikz}
\usepackage{enumitem}
\usepackage{array}
\usepackage{float}
\usepackage{wrapfig}
\usepackage{caption}
\usepackage{subcaption}

\usetikzlibrary{arrows.meta,positioning,fit,calc,shapes.geometric}

\newtheorem{proposition}{Proposition}

\title{Adaptive Anisotropic Attention for Axis-Structured Signals}

\author{%
  Mahir Jain\\
  Mannas AI\\
  \texttt{mahir@mannas.ai}
  \And
  Parshva Runwal\\
  Mannas AI\\
  \texttt{parshva@mannas.ai}
  \And
  Aditya Ray Mishra\\
  Mannas AI\\
  \texttt{aditya@mannas.ai}
    \And
  Raj Singh Bani \\
  Mannas AI\\
  \texttt{raj@mannas.ai}
  \And
  Arvasu Kulkarni\\
  Mannas AI\\
  \texttt{arvasu@mannas.ai}
  \And
  Siddharth Panwar\\
  Mannas AI\\
  \texttt{siddharth@mannas.ai}
  \And
  Sandeep Singh\\
  Mannas AI\\
  \texttt{sandeep@mannas.ai}
}

\begin{document}
\maketitle

\begin{abstract}
Dense self-attention treats all token pairs as equally plausible before learning—an interaction-isotropic prior that can be mismatched to structured signals. For structured, low signal-to-noise ratio (SNR) signals such as EEG, dependencies are organized along the electrode and time axes, and this uniform prior exposes each token to many irrelevant interactions. We introduce Adaptive Anisotropic Attention (AAA), which splits attention into two paths: a temporal path, where each token attends to the tokens of its own electrode across time, and a spatial path, where it attends to the tokens of the other electrodes at the same time step. A small gate predicts, for every token, a convex combination and the token's update is the weighted sum of the two path outputs. On six EEG downstream tasks, the resulting model, AXON (AXis-factorized Operator Network), improves mean balanced accuracy over a dense baseline under both linear probing and full fine-tuning. We show that both paths (temporal and spatial) are necessary and that the weighted sum beats a hard choice of one path; most of the benefit comes from the gate learning a different temporal/spatial balance at each layer of the network. Controlled audio spectrogram experiments show that axis factorization transfers beyond EEG. These results suggest that aligning attention with the natural axes of structured signals provides a useful inductive bias.
\end{abstract}

\section{Introduction}

Transformers~\citep{vaswani2017attention} process tokens: small input
segments represented as vectors. In EEG and audio, tokens form
two-dimensional grids. An EEG token contains one second of signal
from one electrode; an audio token covers one frequency band over
one spectrogram frame. Dense self-attention connects all tokens
directly, allowing long-range interactions without distinguishing
the electrode--time or frequency--time axes.

For structured spatiotemporal signals such as EEG, which lie on electrodes $\times$ time, this interaction-isotropic prior may be mismatched. In masked autoencoding (MAE)~\citep{he2022mae}, the model learns to reconstruct masked patches, not to distinguish downstream classes. Our hypothesis is that dense global attention encourages reconstruction shortcuts: for example, estimating a masked patch from a broad average across electrodes and time. This dilutes localized or axis-specific signals that are critical for downstream tasks such as motor imagery classification.

We introduce Adaptive Anisotropic Attention (AAA), which replaces dense encoder attention with two parallel paths. The temporal path attends within each electrode across time; the spatial path attends across electrodes at the same time step. A small gate combines their outputs through a convex combination. This soft mixture can vary across tokens and layers. We call the resulting model AXON (AXis-factorized Operator Network) and evaluate it on six EEG tasks under linear probing (LP) and full fine-tuning (FT), against dense baselines including a parameter-matched model.

We analyse learned axis mixtures and task-dependent temporal context,
and test the design on audio spectrograms as a second axis-structured
modality. Axis factorization helps
both modalities.

\paragraph{Related work.}
\emph{Factorized attention.} 
Axial attention~\citep{ho2019axial} and divided space-time attention~\citep{bertasius2021space,arnab2021vivit} were built for images and video and attend along one axis at a time. The axis order is set by hand and is the same for every token and every layer, and these designs were not built for, or tested on, low-SNR signals such as EEG. AXON keeps the two axes but runs them in parallel and lets a gate set the balance per token and per layer. 
\emph{EEG foundation models.} EEGPT, LaBraM, CBraMod, BIOT, CSBrain and REVE~\citep{wang2024eegpt,jiang2024labram,wang2025cbramod,yang2023biot,zhou2025csbrain,elouahidi2025reve} pretrain large encoders on unlabelled EEG so that one model transfers to many tasks. They differ in tokenisation and pretraining objective; CBraMod, attends along time and along channels in parallel and combines the two with a fixed split of attention heads. None of them tests whether the dense all-to-all path should be removed, or lets the time/channel balance be learned per token and per layer. We evaluate all six under one protocol (Tables~\ref{tab:external_lp} and~\ref{tab:external_ft}).

\section{Method: Adaptive Anisotropic Attention}
\label{sec:method}

EEG tasks require different temporal and spatial contexts
(Appendix~A). The encoder is a stack of 22 identical blocks, which we call layers; each layer has its own attention paths and its own gate weights, so ``per layer'' below means a separate value in each of the 22 blocks.

\subsection{Input domain}

We process non-overlapping 10-second windows during pretraining and
task-specific windows of 4--30 seconds downstream (Table~\ref{tab:datasets}). Each window
contains $C$ available electrodes and $T$ patches per electrode. A token
$i=(c_i,t_i)$ is a channel--time patch on the grid
$\Omega=\mathcal{C}\times\mathcal{T}$, where $\mathcal{C}$ and
$\mathcal{T}$ index electrodes and temporal patches, respectively.
The full grid contains $CT$ tokens ($231$ for the pretraining grid of $21$ electrodes and $11$ patches). Each electrode has a known
3D head coordinate $p_c\in\mathbb{R}^3$ from the standard 10--20
montage~\citep{jasper195810twenty}, and $x_i\in\mathbb{R}^d$ denotes the token
embedding after patch projection and positional encoding.

\textbf{Dense} attention uses shared $Q,K,V$ projections across
all tokens:
\[
\operatorname{Attn}(x)_i
=
\sum_{j\in\Omega} a_{ij}Vx_j,
\qquad
a_{ij}\propto
\exp\!\left(\frac{q_i^\top k_j}{\sqrt{d_h}}\right),
\]
giving $(CT)^2$ token pairs per layer. \textbf{Dense-L} widens this
baseline to match AXON's parameter count(by increasing dimensions) (Appendix~\ref{app:params}); And \textbf{Divided-ST} applies temporal then spatial attention within each block~\citep{bertasius2021space} (Appendix~\ref{app:divided_st}).

\subsection{Factorized attention}
\label{ssec:factorized}

The core idea is to replace the dense attention operator with a weighted mixture of two axis-restricted operators, each attending within one axis of the token grid $\Omega = \mathcal{C}\times\mathcal{T}$: the temporal path over time within a channel, the spatial path over channels within a time step.

Let $T(x)_i$ denote the temporal attention output and $S(x)_i$ the spatial attention output for token $i$ (defined below). A single AXON block computes:
\[
y_i \;=\; \lambda\, T(x)_i \;+\; (1-\lambda)\, S(x)_i, \qquad \lambda \in (0,1),
\]
where $\lambda$ controls the axis mixture. In the simplest variant (\textbf{AXON-Fixed}), $\lambda = \sigma(\ell)$ is the sigmoid of a single learnable scalar $\ell$ per layer, shared across all tokens. 

The temporal and spatial paths use separate QKV and output projections $(Q^{(T)}, K^{(T)}, V^{(T)}, O^{(T)})$ and $(Q^{(S)}, K^{(S)}, V^{(S)}, O^{(S)})$, allowing each axis to specialise its feature space. The full cost analysis is in Appendix~\ref{app:params}. Although no token sees all others in one layer, on a full channel-time grid any two tokens are connected after two layers: one temporal step and one spatial step (Appendix~\ref{app:axis_diameter}).

\begin{figure}[!htb]
\centering
\begin{subfigure}[c]{0.40\linewidth}
\centering
\resizebox{\linewidth}{!}{
\begin{tikzpicture}[scale=0.5, font=\scriptsize]
\fill[blue!18] (0,3) rectangle (8,4);
\fill[orange!25] (4,0) rectangle (5,6);
\fill[red!75] (4,3) rectangle (5,4);
\draw[step=1, gray!60, thin] (0,0) grid (8,6);
\node[anchor=north] at (4,-0.15) {time patches $t$};
\node[anchor=south, rotate=90] at (-0.35,3) {electrodes $c$};
\node[anchor=west, blue!60!black, align=left] at (8.15,3.5) {$\mathcal{T}(i)$: same channel,\\all time steps};
\node[anchor=south, orange!85!black, align=center] at (4.5,6.15) {$\mathcal{S}(i)$: same time step,\\all channels};
\node[anchor=north west, red!80!black] at (5.05,2.95) {token $i=(c_i,t_i)$};
\end{tikzpicture}
}
\caption{Token grid and the two neighbourhoods.}
\label{fig:grid}
\end{subfigure}\hfill
\begin{subfigure}[c]{0.57\linewidth}
\centering
\resizebox{\linewidth}{!}{
\begin{tikzpicture}[
    font=\small,
    box/.style={draw, rounded corners, align=center, minimum height=0.9cm, minimum width=2.4cm},
    op/.style={draw, rounded corners, align=center, minimum height=0.85cm, minimum width=2.2cm, fill=gray!8},
    gate/.style={draw, rounded corners, align=center, minimum height=0.85cm, minimum width=2.0cm, fill=blue!8},
    arrow/.style={-Latex, thick}
]
\node[box] (input) {Visible EEG tokens\\$(c,t)$ embeddings};

\node[op, below left=1.35cm and 1.9cm of input] (temporal) {Temporal path\\$\mathcal{T}(i)$: same channel\\geometry-free, $T$ tokens};
\node[op, below right=1.35cm and 1.9cm of input] (spatial) {Spatial path\\$\mathcal{S}(i)$: same timestep\\content-based attn, $C$ tokens};
\node[gate, below=1.35cm of input] (gate) {Token gate $g$\\$\mathrm{softmax}(g(\mathrm{sg}(x_i))/\tau)$\\$\Rightarrow (\alpha_i,\beta_i)$};

\node[box, below=2.15cm of gate] (mix) {Axis mixture\\$y_i=\alpha_iT(x)_i+\beta_iS(x)_i$};
\node[box, below=1.1cm of mix] (out) {Residual + LayerNorm + FFN};

\node[draw,dashed,rounded corners,fit=(temporal)(spatial)(gate)(mix),inner sep=0.35cm,label={[align=center]above:AXON block ($\times 22$)}] {};

\draw[arrow] (input) -- (temporal);
\draw[arrow] (input) -- (spatial);
\draw[arrow] (input) -- (gate);
\draw[arrow] (temporal) -- (mix);
\draw[arrow] (spatial) -- (mix);
\draw[arrow] (gate) -- node[right] {$\alpha_i,\beta_i$} (mix);
\draw[arrow] (mix) -- (out);

\node[op, right=3.0cm of mix, fill=red!8] (noglobal) {No dense global path\\$G(x)$ removed\\(MAE shortcut)};
\draw[arrow, dashed] (noglobal) -- (mix);
\end{tikzpicture}
}
\caption{AXON block.}
\label{fig:block}
\end{subfigure}
\caption{
(a) Channel--time grid with query token $i$ (red), temporal neighbours
$\mathcal{T}(i)$ (blue), and spatial neighbours $\mathcal{S}(i)$
(orange); dense attention uses the full grid.
(b) An AXON block mixes both paths using weights from an MLP applied
to $\operatorname{sg}(x_i)$. The encoder stacks 22 blocks without a
dense global branch.
}
\label{fig:arch}
\end{figure}
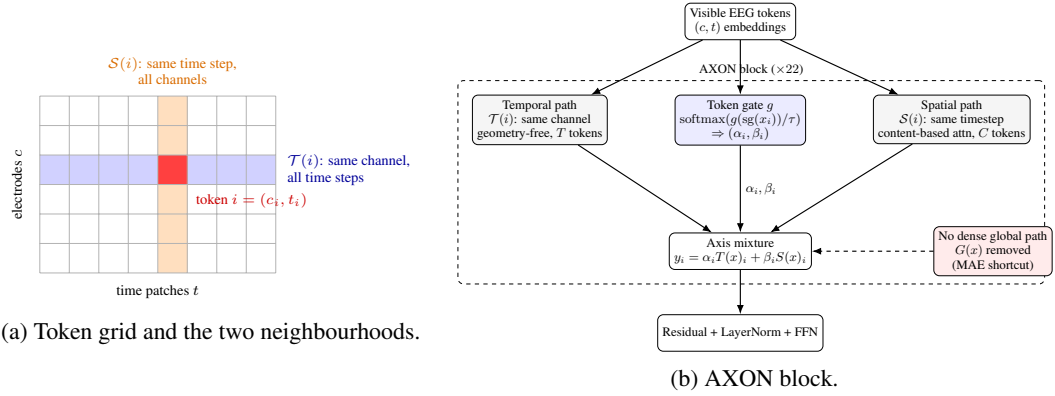

\paragraph{Temporal path.}
The temporal neighbourhood of token $i$ is all tokens on the same channel:
\[
\mathcal{T}(i) = \{j \in \Omega : c_j = c_i\}, \quad |\mathcal{T}(i)| = T.
\]
Each channel group of $T$ tokens is processed as an independent sequence; channels do not interact through the temporal path. This gives each token access to the full temporal context of its electrode: short transients, rhythm-band oscillations, and longer event windows are all contained in $\mathcal{T}(i)$.

\paragraph{Spatial path.}
The spatial neighbourhood of token $i$ is all tokens at the same time step:
\[
\mathcal{S}(i) = \{j \in \Omega : t_j = t_i\}, \quad |\mathcal{S}(i)| = C.
\]
$T(x)_i$ and $S(x)_i$ are standard multi-head attention restricted to these neighbourhoods.

The spatial path is not given any explicit information about electrode coordinates; electrode positions enter only through the positional encoding. (Appendix~\ref{app:ablation_full}).

\subsection{Token-conditioned anisotropy gate}
\label{ssec:gate}
A fixed mixing weight shared by all tokens imposes one mixture on all of them. The token-conditioned gate lets each token choose its own.

\textbf{AXON-TokenGated} replaces the shared layer weight with
token-conditioned mixing:
\[
\begin{aligned}
(\alpha_i,\beta_i)
&=\operatorname{softmax}\!\left(g(\operatorname{sg}(x_i))/\tau\right).
\end{aligned}
\]
The two-layer MLP $g:\mathbb{R}^d\rightarrow\mathbb{R}^2$ has hidden
width $d/4$, GELU activation, and a zero-initialised output layer; $\operatorname{sg}$ refers to stop gradient. We anneal $\tau$ from 2.0 to 1.0 over the first 1500 steps;
higher values keep the weights closer to equal (Appendix~C).

The block output is the soft-gated mixture of the temporal and spatial paths:
\[
y_i = \alpha_i\, T(x)_i + \beta_i\, S(x)_i, \qquad \alpha_i + \beta_i = 1,
\]
where $T$ and $S$ are the temporal and spatial paths defined in Section~\ref{ssec:factorized}. Each layer computes a gate for every token on every forward pass.
Because $x_i$ includes positional encoding, weights can vary with
content, channel--time position, layer, and input sample.
AXON-TokenGated uses one full-length temporal window and no global
path; final AXON replaces $T$ with the two-window path (Section~2.5).
Gate interventions are in Section~3.2 and Appendix~H.

\subsection{Global path}
\label{ssec:global}

A natural extension adds a third path: dense attention over all tokens, with its own projections. The gate then predicts three weights that sum to one:
\[
y_i = \alpha_i\, T(x)_i + \beta_i\, S(x)_i + \gamma_i\, G(x)_i, \qquad \alpha_i + \beta_i + \gamma_i = 1,
\]
where $G(x)_i = \sum_{j \in \Omega} g_{ij}\, V_G x_j$. This is \textbf{AXON-withGlobal}. The motivation was that clinical tasks may benefit from direct all-to-all context in a single layer, which the two axis paths only reach after two layers. AXON does not use this path; Section~\ref{sec:results} reports its effect.

\subsection{Two temporal windows}
\label{ssec:multiscale}

EEG carries information at different time scales: a motor-imagery response or a seizure onset develops over a few seconds, while a sleep stage or a cognitive state lasts the whole window~\citep{pfurtscheller1999erders}.
 A temporal path that always sees the whole window can in principle use both, but it has to learn to separate a short event from the slow background on its own. We make the separation explicit by giving the temporal path two windows: a short one that sees only five consecutive patches, and a long one that sees the whole visible window. Each layer learns how much to use each:
\[
T_{\mathrm{ms}}(x)_i = \lambda^s\, T_{\mathrm{short}}(x)_i + (1-\lambda^s)\, T_{\mathrm{long}}(x)_i, \qquad \lambda^s \in [0,1],
\]
with one $\lambda^s$ per layer, initialised to favour the long window so that early reconstruction is easy. This choice is separate from the axis gate: the axis gate sets the temporal/spatial split for each token, and $\lambda^s$ sets how far in time the temporal path looks. \textbf{AXON}, the final model, is AXON-TokenGated with this two-window temporal path and no global path.

\section{Experiments}

\subsection{Setup}

All encoders are pretrained with masked autoencoding~\citep{he2022mae} on pooled TUH-EEG~\citep{obeid2016tuh}, I-CARE~\citep{amorim2023icareeeg}, and internal EEG data for 50 epochs with batch size 4096. Recordings are mapped to 21 canonical 10--20 electrode positions and divided into 10-second windows with $T=11$ patches per channel. We mask 55\% of tokens; the encoder processes only visible tokens, and a small decoder reconstructs masked patches. Missing channels are excluded from tokenization and reconstruction loss; attention masks cover only present tokens.

We evaluate on six public datasets covering motor imagery (motor, bcic), cognitive workload (workload), sleep staging (hmc), seizure detection (siena), and dementia diagnosis (adftd). We report balanced accuracy (BAC) on subject-disjoint splits using \textbf{linear probing} (LP; shallow MLP head) and \textbf{full fine-tuning} (FT). For each model, Table~\ref{tab:main_results} reports the checkpoint selected by the highest mean LP on held-out validation subjects, rather than the final epoch. Dataset summaries appear in Table~\ref{tab:datasets}; implementation, training-budget analysis, and evaluation details are in Appendices~\ref{app:implementation} and~\ref{app:downstream}.

\subsection{Results}\label{sec:results}
\label{ssec:results}

\begin{table}[!t]
\centering
\caption{Main EEG results --- \textbf{Linear Probe} (LP, frozen encoder) and \textbf{Full Finetune} (FT). Balanced accuracy, mean $\pm$ std over 3 downstream seeds; subject-disjoint splits shared across all models.}
\label{tab:main_results}
\resizebox{\linewidth}{!}{
\begin{tabular}{lccccccc}
\toprule
Model & motor & workload & hmc & siena & adftd & bcic & Mean LP \\
\midrule
Dense       & $0.353{\pm}.002$ & $0.612{\pm}.021$ & $0.660{\pm}.003$ & $\mathbf{0.876{\pm}.007}$ & $0.516{\pm}.034$ & $0.282{\pm}.008$ & 0.550 \\
Dense-L & $0.384{\pm}.003$ & $0.648{\pm}.011$ & $0.653{\pm}.002$ & $0.828{\pm}.000$ & $0.531{\pm}.004$ & $0.288{\pm}.001$ & 0.555 \\
Divided-ST & $0.393{\pm}.001$ & $0.621{\pm}.008$ & $0.649{\pm}.002$ & $0.872{\pm}.003$ & $0.523{\pm}.028$ & $\mathbf{0.299{\pm}.002}$ & 0.560 \\
AXON-Fixed           & $0.370{\pm}.003$ & $\mathbf{0.674{\pm}.007}$ & $0.666{\pm}.000$ & $0.841{\pm}.001$ & $0.502{\pm}.008$ & $0.296{\pm}.003$ & 0.558 \\
AXON-TokenGated & $0.447{\pm}.001$ & $0.662{\pm}.012$ & $\mathbf{0.673{\pm}.001}$ & $0.855{\pm}.000$ & $0.531{\pm}.014$ & $0.298{\pm}.004$ & 0.578 \\
\textbf{AXON}         & $\mathbf{0.455{\pm}.005}$ & $\mathbf{0.673{\pm}.009}$ & $0.650{\pm}.002$ & $0.867{\pm}.001$ & $\mathbf{0.538{\pm}.006}$ & $0.292{\pm}.007$ & \textbf{0.579} \\
\bottomrule
\end{tabular}
}
\vspace{4pt}
\resizebox{\linewidth}{!}{
\begin{tabular}{lccccccc}
\toprule
\multicolumn{8}{l}{\textit{Full Finetune}} \\
\midrule
Model & motor & workload & hmc & siena & adftd & bcic & Mean FT \\
\midrule
Dense       & $0.561{\pm}.015$ & $0.648{\pm}.008$ & $\mathbf{0.736{\pm}.003}$ & $0.853{\pm}.001$ & $0.504{\pm}.046$ & $0.335{\pm}.038$ & 0.606 \\
Dense-L & $0.623{\pm}.001$ & $0.613{\pm}.012$ & $0.728{\pm}.005$ & $0.826{\pm}.009$ & $0.536{\pm}.011$ & $\mathbf{0.469{\pm}.051}$ & 0.633 \\
Divided-ST & $0.600{\pm}.006$ & $0.661{\pm}.057$ & $0.720{\pm}.003$ & $0.853{\pm}.011$ & $0.555{\pm}.021$ & $0.379{\pm}.017$ & 0.628 \\
AXON-Fixed           & $0.619{\pm}.009$ & $\mathbf{0.735{\pm}.008}$ & $0.731{\pm}.006$ & $0.853{\pm}.013$ & $0.501{\pm}.029$ & $0.403{\pm}.020$ & 0.640 \\
AXON-TokenGated & $\mathbf{0.630{\pm}.012}$ & $0.670{\pm}.033$ & $0.729{\pm}.003$ & $\mathbf{0.867{\pm}.013}$ & $0.563{\pm}.005$ & $0.399{\pm}.046$ & 0.643 \\
\textbf{AXON}         & $0.625{\pm}.005$ & $0.685{\pm}.006$ & $0.732{\pm}.009$ & $0.859{\pm}.007$ & $\mathbf{0.604{\pm}.022}$ & $0.428{\pm}.033$ & \textbf{0.656} \\
\bottomrule
\end{tabular}
}
\end{table}

AXON achieves the highest Mean LP (0.579) and Mean FT (0.656)
in Table~\ref{tab:main_results}. It improves LP and FT over Dense
by 2.9 and 5.0 balanced-accuracy points, respectively. The gains
remain over Dense-L (2.4 and 2.3 points) and Divided-ST
(1.9 and 2.8 points). The largest task-level gains over Dense
are on motor LP (+10.2 points) and adftd FT (+10.0 points).

Three further controls show that the gain is not simply from attending to fewer tokens, and not simply from position information: restricting the spatial path to each electrode's nearest neighbours drops Mean LP to 0.510; the trained dense model still attends to pairs that share neither electrode nor time step (Table~\ref{tab:relation_mass}); and adding 2D rotary position embeddings~\citep{su2024roformer} to the dense model gains only $+$0.008 (Appendix~\ref{app:ablation_full}). The ranking is stable at every pretraining budget we tested (Appendix~\ref{app:budget}). A variant that biases the temporal path toward nearby time steps (AXON-Decay) matches AXON on Mean LP but lowers Mean FT (Appendix~\ref{app:temporal_locality_full}). Appendix~\ref{app:gate_mechanism} shows what the gate learns and tests whether the model depends on it. AXON-withGlobal, the variant with a third dense path (Section~\ref{ssec:global}), underperforms both AXON and AXON-Fixed on Mean LP and FT (Table~\ref{tab:ablation}) despite reaching comparable reconstruction loss (Figure~\ref{fig:recon_loss}). Its first-layer global weight is $\gamma=0.724$, leaving about a quarter for the axis paths. We hypothesise that global averaging provides a reconstruction shortcut that weakens the learning of axis-specific features (Appendix~\ref{app:global_shortcut}). 

\paragraph{Comparison with released EEG foundation models.}
Under the same six-task protocol, AXON has the highest mean scores among six released EEG foundation models: 0.579 vs.\ 0.527 Mean LP and 0.656 vs.\ 0.614 Mean FT against the next-best, REVE (Appendix~\ref{app:external_baseline_details}). 

\section{Audio Spectrogram Experiments}
We apply the same axis-factorization principle to time–frequency spectrograms, comparing four attention variants under the AudioMAE pretraining recipe~\citep{huang2022audiomae}.  A spectrogram is a grid too: one axis is time, the other is frequency, and a token is one frequency band over one time frame. The temporal path attends across time within a frequency band; the spatial path becomes a frequency path and attends across frequency bands within a time frame. The full-scale experiment uses AudioSet-2M corpus~\citep{gemmeke2017audioset} ($\sim$2M clips); smaller-scale results and the comparison with published AudioMAE are in Appendix~\ref{app:audio}. Downstream evaluation uses AudioSet, ESC-50~\citep{piczak2015esc} and SpeechCommands (SC)~\citep{warden2018speechcommands}; AudioSet is multi-label, so we report mean average precision (mAP), and the other two report accuracy.

\begin{table}[!t]
\centering
\caption{Controlled audio spectrogram results --- full AudioSet-2M pretraining. All variants use identical optimizer, schedule, and compute budget; only the attention operator differs.}
\label{tab:audio}
\resizebox{\linewidth}{!}{
\begin{tabular}{lcccc}
\toprule
Model & AudioSet FT mAP & AudioSet LP mAP & ESC-50 LP Acc & SC LP Acc \\
\midrule
Dense & $11.59{\pm}0.25$ & $4.08{\pm}0.01$ & $45.25{\pm}0.94$ & $19.53{\pm}0.07$ \\
Fixed gate (AXON-Fixed) & $13.75{\pm}0.16$ & $4.01{\pm}0.02$ & $44.83{\pm}0.31$ & $\mathbf{22.01{\pm}0.06}$ \\
Token gate (AXON-TokenGated) & $13.82{\pm}0.25$ & $4.84{\pm}0.03$ & $45.75{\pm}0.71$ & $21.96{\pm}0.15$ \\
Token + Global (AXON-withGlobal) & $\mathbf{14.29{\pm}0.05}$ & $\mathbf{5.08{\pm}0.04}$ & $\mathbf{47.17{\pm}0.92}$ & $21.87{\pm}0.16$ \\
\bottomrule
\end{tabular}
}
\end{table}

All factorized variants improve over Dense on AudioSet FT and SpeechCommands LP (Table~\ref{tab:audio}). Unlike EEG, Token+Global achieves the highest scores on three of the four metrics. On AudioSet FT and SpeechCommands LP this advantage holds at all three pretraining scales we tested, 18K, 200K and 2M clips (Tables~\ref{tab:audio_18k} and~\ref{tab:audio_200k}). The audio experiment therefore extends the factorization result while showing that the global branch is not uniformly detrimental across the tested settings.

\section{Conclusion}

AXON combines temporal and spatial attention through learned soft mixing. Across six subject-disjoint EEG tasks, it improves mean balanced accuracy over dense and parameter-matched dense baselines under Linear Probing and full fine-tuning. Gate interventions show larger performance drops from removing either axis or using hard routing than from replacing token gates with layer means. Axis factorization extends to audio, a second axis-structured modality. These results support a simple design rule for structured signals: align attention with the signal's axes and let the model learn how much to use each.

\newpage
\bibliographystyle{plainnat}
\bibliography{references}

\newpage
\section*{Appendix}
\appendix 

\section{Background \& Domain Context}
\label{app:background}

Electroencephalography (EEG) measures electrical potential differences across a sparse array of scalp electrodes at millisecond resolution. Each recording is a matrix of $C$ channels by $T$ time samples. The signal is low-amplitude ($\sim$10--100\,$\mu$V), contaminated by muscle artifacts, eye movements, and ambient noise, and varies across subjects and sessions. Despite this noise, EEG encodes clinically and cognitively meaningful structure: motor imagery induces localized mu-rhythm desynchronization over sensorimotor cortex; sleep staging depends on broadband slow-wave and spindle patterns; epileptic events manifest as sharp, spatially propagating discharges.

Two properties make EEG a natural proving ground for anisotropic attention. First, the token is explicitly two-dimensional: a patch token at $(c,t)$ has a known spatial identity (electrode $c$ with head coordinate) and a temporal identity (time window $t$). Second, different tasks depend on structurally different contexts: motor imagery requires preserving localized, lateralized, temporally precise structure; clinical classification benefits from broader spatial and temporal integration. A single isotropic attention operator may be a poor shared prior across tasks with such different channel-time structure.

\section{External EEG Foundation Models}
\label{app:external_baseline_details}

\subsection{External EEG baselines --- Linear Probe}
\label{app:external_lp}
We compare AXON against six published EEG foundation models under a standardised evaluation protocol: the same six tasks, the same disjoint subject splits, the same balanced-accuracy metric, and the same three downstream seeds as our internal ablations. All external models are evaluated from their officially released pretrained checkpoints. Models are evaluated with both linear probing (LP, frozen encoder) and fine-tuning (FT, all weights updated).

\begin{table}[h]
\centering
\caption{External EEG foundation model comparison --- \textbf{Linear Probe} (LP, frozen encoder). Format: balanced accuracy. \textbf{Bold} = column best. Our dense baseline is included as a within-setup reference. All models evaluated under the same 6-task protocol with identical subject splits.}
\label{tab:external_lp}
\resizebox{\linewidth}{!}{
\begin{tabular}{lccccccc}
\toprule
Model & motor & workload & hmc & siena & adftd & bcic & Mean \\
\midrule
EEGPT~\citep{wang2024eegpt} & $0.381{\pm}.015$ & $0.574{\pm}.020$ & $\mathbf{0.665{\pm}.006}$ & $0.804{\pm}.019$ & $0.393{\pm}.034$ & $0.281{\pm}.007$ & 0.516 \\
LaBraM~\citep{jiang2024labram} & $0.268{\pm}.012$ & $0.500{\pm}.000$ & $0.381{\pm}.019$ & $0.500{\pm}.000$ & $0.309{\pm}.019$ & $0.285{\pm}.022$ & 0.374 \\
CBraMod~\citep{wang2025cbramod} & $0.259{\pm}.013$ & $0.500{\pm}.000$ & $0.510{\pm}.001$ & $0.619{\pm}.006$ & $0.358{\pm}.002$ & $0.270{\pm}.010$ & 0.419 \\
BIOT~\citep{yang2023biot} & $0.284{\pm}.008$ & $0.577{\pm}.081$ & $0.644{\pm}.003$ & $0.610{\pm}.025$ & $0.492{\pm}.033$ & $0.261{\pm}.010$ & 0.478 \\
CSBrain~\citep{zhou2025csbrain} & $0.272{\pm}.005$ & $0.500{\pm}.000$ & $0.568{\pm}.002$ & $0.500{\pm}.000$ & $0.364{\pm}.014$ & $0.270{\pm}.009$ & 0.412 \\
REVE~\citep{elouahidi2025reve} & $0.315{\pm}.003$ & $\mathbf{0.709{\pm}.035}$ & $0.653{\pm}.006$ & $0.688{\pm}.033$ & $0.532{\pm}.052$ & $0.267{\pm}.012$ & 0.527 \\
\midrule
Dense & $0.353{\pm}.002$ & $0.612{\pm}.021$ & $0.660{\pm}.003$ & $\mathbf{0.876{\pm}.007}$ & $0.516{\pm}.034$ & $0.282{\pm}.008$ & 0.550 \\
\textbf{AXON} & $\mathbf{0.455{\pm}.005}$ & $0.673{\pm}.009$ & $0.650{\pm}.002$ & $0.867{\pm}.001$ & $\mathbf{0.538{\pm}.006}$ & $\mathbf{0.292{\pm}.007}$ & \textbf{0.579} \\
\bottomrule
\end{tabular}
}
\end{table}

AXON achieves the highest Mean LP and Mean FT among all eight models, ahead of the next-best external model (REVE) by $+$5.2 points LP ($+$9.9\% relative) and $+$4.2 points FT ($+$6.8\% relative). The full fine-tuning comparison is in Appendix Table~\ref{tab:external_ft}. Note that our own dense baseline already outperforms all six external models on mean LP. So part of AXON's gap to the external models comes from our training setup rather than from the attention design, and AXON's gain over that dense baseline (Table~\ref{tab:main_results}) is measured on top of it. The architectural claim rests on that controlled comparison, not on this table.

Task-level analysis reveals that motor imagery exhibits the largest gap between AXON and external models ($+$19\% LP over the best competitor). This fits the motivation for preserving axis structure: motor imagery depends on a left/right difference in mu and beta rhythms over the sensorimotor electrodes. This comparison does not, however, show which features account for the gain. REVE retains advantages on workload and siena FT, suggesting its architecture provides broader temporal integration suited for sustained cognitive states.

\subsection{External EEG baselines --- Full Finetune}
\label{app:external_ft}

\begin{table}[h]
\centering
\caption{External EEG foundation model comparison --- \textbf{Full Finetune} (FT). \textbf{Bold} = column best. All models evaluated under the same 6-task protocol with identical subject splits.}
\label{tab:external_ft}
\resizebox{\linewidth}{!}{
\begin{tabular}{lccccccc}
\toprule
Model & motor & workload & hmc & siena & adftd & bcic & Mean \\
\midrule
EEGPT~\citep{wang2024eegpt} & $0.513{\pm}.007$ & $0.668{\pm}.015$ & $0.712{\pm}.005$ & $0.795{\pm}.031$ & $0.404{\pm}.022$ & $0.280{\pm}.041$ & 0.562 \\
LaBraM~\citep{jiang2024labram} & $0.249{\pm}.001$ & $0.501{\pm}.002$ & $0.645{\pm}.011$ & $0.813{\pm}.018$ & $0.290{\pm}.060$ & $0.259{\pm}.009$ & 0.460 \\
CBraMod~\citep{wang2025cbramod} & $0.426{\pm}.015$ & $0.554{\pm}.044$ & $0.709{\pm}.011$ & $0.847{\pm}.036$ & $0.319{\pm}.026$ & $0.303{\pm}.031$ & 0.526 \\
BIOT~\citep{yang2023biot} & $0.372{\pm}.013$ & $0.570{\pm}.092$ & $0.710{\pm}.005$ & $0.747{\pm}.026$ & $0.449{\pm}.038$ & $0.314{\pm}.040$ & 0.527 \\
CSBrain~\citep{zhou2025csbrain} & $0.570{\pm}.023$ & $0.605{\pm}.028$ & $0.705{\pm}.010$ & $0.782{\pm}.036$ & $0.432{\pm}.021$ & $0.350{\pm}.015$ & 0.574 \\
REVE~\citep{elouahidi2025reve} & $0.612{\pm}.004$ & $\mathbf{0.703}{\pm}.011$ & $0.724{\pm}.002$ & $\mathbf{0.863}{\pm}.033$ & $0.460{\pm}.050$ & $0.322{\pm}.014$ & 0.614 \\
\midrule
Dense & $0.561{\pm}.015$ & $0.648{\pm}.008$ & $\mathbf{0.736}{\pm}.003$ & $0.853{\pm}.001$ & $0.504{\pm}.046$ & $0.335{\pm}.038$ & 0.606 \\
\textbf{AXON} & $\mathbf{0.625}{\pm}.005$ & $0.685{\pm}.006$ & $0.732{\pm}.009$ & $0.859{\pm}.007$ & $\mathbf{0.604}{\pm}.022$ & $\mathbf{0.428}{\pm}.033$ & \textbf{0.656} \\
\bottomrule
\end{tabular}
}
\end{table}

\section{Extended Implementation Details \& Pretraining}
\label{app:implementation}

\paragraph{Pretraining corpus.}
The pretraining corpus pools four data sources (Table~\ref{tab:pretrain_corpus}). Two are publicly available: the Temple University Hospital EEG Corpus (TUH-EEG)~\citep{obeid2016tuh} and the I-CARE dataset~\citep{amorim2023icareeeg}. Two are internal clinical EEG collections acquired under institutional ethics approval and de-identified before use; these are not publicly released but are described below to enable reproducibility assessment.

\begin{table}[h]
\centering
\caption{Pretraining corpus composition. All sources are pooled into a single unlabeled pretraining set; no downstream task labels are used during pretraining. Internal sources are marked~$\dagger$.}
\label{tab:pretrain_corpus}
\small
\begin{tabular}{lrrll}
\toprule
Source & Subjects & Hours & Clinical context & Availability \\
\midrule
TUH-EEG~\citep{obeid2016tuh} & ${\sim}$15{,}000 & ${\sim}$25{,}000 & Mixed clinical referrals & Public \\
I-CARE~\citep{amorim2023icareeeg} & ${\sim}$600 & ${\sim}$33{,}000 & Post-cardiac-arrest ICU & Public \\
Internal-A$\,^\dagger$ & 4{,}539 & 2{,}546 & Routine clinical neurophysiology & Not released \\
Internal-B$\,^\dagger$ & 1{,}050 & 435 & Multi-centre research EEG & Not released \\
\bottomrule
\end{tabular}
\end{table}

\noindent\textit{Internal-A} comprises routine clinical EEG recordings (resting-state, hyperventilation, and photic stimulation protocols) collected across hospital neurophysiology departments. \textit{Internal-B} comprises multi-centre research EEG recordings acquired under a national research programme. Both internal datasets were recorded with standard 10-20 montage systems at sampling rates of 250--512\,Hz and de-identified (all patient identifiers, dates, and institution codes removed) before inclusion. All internal data collection was conducted under institutional ethics board approval with informed consent or waiver of consent for retrospective de-identified use. Subject identities are verified disjoint across all four pretraining sources and all six downstream evaluation datasets.

Recordings span diverse acquisition settings with \textbf{variable electrode configurations}: systems range from compact 16-channel ambulatory devices to full 256-channel research amplifiers, and not every recording contains all standard 10-20 electrodes. We retain only recordings whose channel header resolves to a subset of the international 10-20 montage, then extract the available 10-20 electrode positions per recording. Because channel count varies across sources, we adopt $C\!=\!21$ as the representative value throughout this paper; this is the mode channel count across the retained pretraining corpus. Preprocessing: (1)~resample to $f_s\!=\!200$\,Hz; (2)~notch filter at 50 and 60\,Hz; (3)~bandpass [0.5, 99.5]\,Hz; (4)~per-channel z-score normalisation; (5)~clip at ${\pm}15\sigma$; (6)~segment into non-overlapping 10-second windows.

We train encoders using masked autoencoding (MAE). EEG signal is divided into patches of 200 samples (1\,s at $f_s\!=\!200$\,Hz) with a 20-sample overlap and 180-sample (0.9\,s) stride between consecutive patch start positions, yielding $T\!=\!11$ patches per channel per 10-second window. Each channel-time patch becomes a single token via a learnable linear projection. Tokens receive a split positional encoding: spatial coordinates $p_c = (x,y,z)$ (3D head positions in millimetres, standard 10-20 montage) are projected with a learned linear layer to produce $\mathrm{PE}_S$; the temporal patch index receives a fixed sinusoidal encoding $\mathrm{PE}_T$. The two components are summed: $\mathrm{PE}(i) = \mathrm{PE}_S(c_i) + \mathrm{PE}_T(t_i)$. The encoder processes only the \textbf{visible} tokens (55\% masking ratio, spatiotemporal block masking with spatial radius 3.0 and temporal radius 3.0); masked token positions receive no encoder gradient. A lightweight 4-layer dense Transformer decoder takes the encoded visible tokens plus learned mask-slot embeddings and reconstructs all patches. Training minimises L1 reconstruction loss over masked patches plus an auxiliary pooled-attention reconstruction loss ($\lambda\!=\!0.5$). The auxiliary head applies cross-attention pooling over the concatenated outputs of all encoder MHA layers: a single learned query token attends over the layer-wise output tokens to produce a compact global representation. This pooled token is then repeated to match the number of masked positions, enriched with positional encodings, and passed through a 2-layer FFN to reconstruct the masked patches under a separate L1 loss. The total pretraining loss is $\mathcal{L} = \mathcal{L}_{\text{primary}} + \lambda \cdot \mathcal{L}_{\text{aux}}$.

\subsection{Pretraining hyperparameters}
\label{app:hyperparameters}

\begin{table}[h]
\centering
\caption{Pretraining hyperparameters (AXON). The dense baseline uses the same schedule and corpus; it differs only in the encoder attention operator.}
\label{tab:pretrain_hyp}
\small
\begin{tabular}{ll}
\toprule
Hyperparameter & Value \\
\midrule
Batch size & 4096 \\
Epochs & 50 \\
Peak LR & $2.4\!\times\!10^{-4}$ \\
LR schedule & CosineAnnealingLR ($T_{\max}\!=\!20$); \\
Optimiser & fused AdamW ($\beta_1\!=\!0.9$, $\beta_2\!=\!0.95$, $\lambda\!=\!0.05$) \\
Gradient clip ($\ell_2$ norm) & 1.0 \\
Gate $\tau$ warmup & $2.0\!\to\!1.0$ over 1500 steps \\
Mask ratio & 0.55 (spatiotemporal block masking) \\
Spatial block radius & 3.0 patch indices (i.e., 3 electrode positions in the canonical 10-20 ordering) \\
Temporal block radius & 3.0 patch indices (i.e., 3 consecutive time patches, $\approx$2.7\,s) \\
Dropout mask ratio & 0.3 \\
Coordinate noise $\sigma$ & 0.25 \\
Auxiliary loss weight & 0.5 \\
Model dim $d$ & 512 \\
Encoder layers / heads & 22 / 8 \\
Decoder layers & 4 (dense, full attention) \\
Precision & bfloat16 (autocast) \\
\bottomrule
\end{tabular}
\end{table}
Pretraining was conducted on a single compute node equipped with 8 NVIDIA H200 GPUs (143,771 MiB memory each, $\approx$1.1 TiB total GPU memory) and $\approx$2.2 TiB of host RAM.
\subsection{Parameter count and compute}
\label{app:params}

AXON is not parameter-matched to the dense baseline. Replacing one dense attention operator with two separate full-width temporal and spatial operators adds approximately 26M parameters (22.5\%).

\paragraph{Independent projections.}
The temporal and spatial paths use separate QKV and output projection matrices $(Q^{(T)}, K^{(T)}, V^{(T)}, O^{(T)})$ and $(Q^{(S)}, K^{(S)}, V^{(S)}, O^{(S)})$. This is deliberate: the features needed to select which temporal patch to attend to (e.g., spectral power at a rhythm band) are different from those needed to select which electrode to attend to (e.g., lateralised activation). Shared projections would force both axes through the same feature bottleneck, limiting specialisation. The cost is a $2\times$ parameter increase in the attention projections per layer relative to a single dense attention, but the number of pairwise attention interactions is reduced by a factor of $\approx(CT)/(C+T)$ (from $(CT)^2$ to $CT^2+TC^2$). 

To control for this, we trained a parameter-matched dense baseline with hidden dimension 568 (142.5M total parameters, within 0.5\% of AXON's 141.9M). This model uses identical pretraining (same data, optimizer, epochs, mask ratio) and differs only in hidden dimension.

\begin{table}[h]
\centering
\caption{Parameter count, attention complexity, and downstream performance. Dense-L controls for the capacity difference by matching AXON's total parameter count.}
\label{tab:params}
\small
\begin{tabular}{lcccc}
\toprule
Model & Total params & Attn ops/layer$^\ddagger$ & Mean LP & Mean FT \\
\midrule
Dense     & 115.9M & $O((CT)^2) = 53{,}361$ & 0.550 & 0.606 \\
Dense-L & 142.5M & $O((CT)^2) = 53{,}361$ & 0.555 & 0.633 \\
AXON-withGlobal         & $\sim$167M & $> (CT)^2$ & 0.555 & 0.619 \\
\textbf{AXON} & \textbf{141.9M} & $O(CT^2+TC^2)=7{,}392$ & \textbf{0.579} & \textbf{0.656} \\
\bottomrule
\end{tabular}
\smallskip

\small$^\ddagger$Computed for $C=21$ channels, $T=11$ time patches (mode of pretraining corpus): Dense $(CT)^2=53{,}361$; AXON $CT^2+TC^2=7{,}392$ (${\approx}7\times$ less).
\end{table}

Three observations address the parameter concern. First, the parameter-matched dense baseline (dim=568, 142.5M) uses nearly identical capacity to AXON (141.9M) with the same dense attention topology. It achieves Mean LP 0.555 and Mean FT 0.633: higher than the original dense baseline (0.550 / 0.606), demonstrating that the extra parameters provide some benefit, but still falling short of AXON by 2.4 points on Mean LP and 2.3 points on Mean FT ($+$4.3\% and $+$3.6\% relative). The AXON advantage persists after capacity matching. Second, AXON replaces global quadratic mixing with two axis-factorized operators, reducing pairwise attention interactions by $\sim$7$\times$ per layer despite the added projections. Because the sequence length ($N\!=\!231$) is small relative to $d\!=\!512$, the $O(Nd^2)$ projection costs dominate total FLOPs; the computational advantage of axis factorization therefore lies in the structured prior, not in raw speed. Third, the AXON-withGlobal variant has substantially more parameters (${\sim}$167M, adding a global QKV on top of temporal and spatial branches) and still underperforms AXON by 2.4 points on Mean LP and 3.7 points on Mean FT. If the gains were explained by parameter count, AXON-withGlobal should win. The parameter-matched dense and AXON-withGlobal comparisons together isolate the inductive bias, not a capacity effect.

\paragraph{Empirical gate statistics.}
The gate diagnostic (22 layers $\times$ 6 datasets) shows that the trained gate is neither trivially uniform nor collapsed. Mean axis weights: $\alpha_\text{mean}=0.441$ (temporal), $\beta_\text{mean}=0.559$ (spatial). Mean gate entropy ratio: 0.932 (scale 0--$\log 2$, where 1.0 is fully uniform). No layer falls below the collapse threshold of 0.40. Gate intervention experiments (Table~\ref{tab:gate_interventions}) show that the dominant learned structure is a soft depth-dependent anisotropy schedule (layer-mean gates drop only $-$2.1\% vs learned), while exact per-token gate assignment is a secondary effect (shuffled gates drop only $-$1.3\%).

\paragraph{Stop-gradient.}\label{app:stopgrad}
The stop-gradient on the gate input means the encoder receives no gradient from the routing decision; it is trained only by the reconstruction loss. We included it as a precaution so that the gate could not reshape the encoder's representations during training. The ablation shows it makes no measurable difference: removing it, so that the gate reads the live representation instead of a detached copy~\citep{chen2021simsiam}, changes Mean LP by only $-$0.002 (0.577 vs.\ 0.579). It is a safe default, not a source of gain; our reported results do not hinge on it.

\paragraph{Temperature annealing.}
The gate softmax is divided by a temperature $\tau$ annealed from $\tau_\mathrm{start}=2.0$ to $\tau_\mathrm{end}=1.0$ over the first 1500 training steps:
\[
\tau_t = \tau_\mathrm{start} + \left(\tau_\mathrm{end} - \tau_\mathrm{start}\right) \cdot \min\!\left(1, \frac{t}{1500}\right).
\]
During warmup the gate is soft, allowing both axes to receive gradient from the MAE reconstruction loss. Both branches therefore develop useful representations before the gate sharpens. Gate entropy is monitored throughout warmup; a drop below 0.40 before warmup ends would indicate premature routing commitment and would be corrected by increasing $\tau_\mathrm{start}$ or extending the warmup window.

\section{Downstream Tasks \& Evaluation Protocol}
\label{app:downstream}

We evaluate with two protocols: \textbf{linear probing} (LP), where the encoder is frozen and only the classification head is trained, and \textbf{full finetuning} (FT), where all encoder weights are updated. The classification head is: \texttt{AdaptiveAvgPool1d} $\to$ \texttt{Linear}(512, 128) $\to$ ELU $\to$ \texttt{Dropout}(0.3) $\to$ \texttt{Linear}(128, $K$), where $K$ is the number of classes. The downstream optimiser is AdamW (weight decay $0.01$) with cosine-annealing LR (peak $2\!\times\!10^{-4}$, min $2\!\times\!10^{-5}$) preceded by a 5-epoch linear warmup from $2\!\times\!10^{-6}$. 30 training epochs. The metric is \textbf{balanced accuracy} throughout.  Subject splits are disjoint across train, validation, and test for every dataset; the same splits are shared by all models.

\begin{table}[h]
\centering
\caption{Downstream evaluation datasets. All splits are subject-disjoint.}
\label{tab:datasets}
\small
\resizebox{\linewidth}{!}{
\begin{tabular}{lllcll}
\toprule
Dataset & Task & Classes & Window & Train split & Val / Test split \\
\midrule
motor\_mv\_img~\citep{schalk2004bci2000,goldberger2000physionet} & Motor imagery (L/R/both/foot) & 4 & 4\,s & Subj.~0--69 & Subj.~70--88 / 89--109 \\
bcic\_2a~\citep{tangermann2012bcicompiv} & Motor imagery BCI (4 limbs) & 4 & 4\,s & Subj.~1--5 & Subj.~6--7 / 8--9 \\
workload~\citep{lim2018stew} & Cognitive workload (low/high) & 2 & 4\,s & $\approx$72\% subj. & $\approx$14\% / 14\% subj. \\
hmc~\citep{alvarezestevez2022hmc} & Sleep staging (5 stages) & 5 & 30\,s & $\approx$67.6\% subj. & $\approx$16.2\% / 16.2\% subj. \\
siena\_scalp~\citep{detti2020siena} & Seizure detection (ictal/interictal) & 2 & 10\,s & $\approx$70\% subj. & $\approx$15\% / 15\% subj. \\
adftd~\citep{miltiadous2023adftd} & Dementia (AD/FTD/Healthy) & 3 & 10\,s & $\approx$70\% subj. & $\approx$15\% / 15\% subj. \\
\bottomrule
\end{tabular}
}
\end{table}

\subsection{Downstream tasks}

We evaluate across six datasets covering BCI, cognitive, and clinical tasks. All evaluations use disjoint subject splits and balanced accuracy. Table~\ref{tab:downstream} summarises the key discriminative challenge of each task and explains why isotropic attention is an unfavourable inductive bias.

\begin{table}[h]
\centering
\caption{Downstream evaluation tasks.}
\label{tab:downstream}
\small
\begin{tabular}{llcl}
\toprule
Dataset & Family & Classes & Key discriminative structure \\
\midrule
motor\_mv\_img & BCI & 4 & Lateralized mu/beta ERD at movement onset \\
bcic\_2a & BCI & 4 & Fine lateralization, 4 limb classes \\
workload & Cognitive & 2 & Sustained frontal-parietal synchrony\\
hmc & Clinical & 5 & Broadband spectral stage transitions \\
siena\_scalp & Clinical & 2 & Spatially propagating ictal discharge \\
adftd & Clinical & 3 & Diffuse cortical slowing, theta excess \\
\bottomrule
\end{tabular}
\end{table}

Motor imagery tasks are most sensitive to interaction-isotropic mixing: mu/beta ERD lateralization~\citep{pfurtscheller1999erders} (left vs.~right hand) is the primary discriminative signal, and averaging across all tokens via dense global attention can suppress this asymmetry.

\subsection{Training budget}
\label{app:budget}
Table~\ref{tab:main_results} reports the best-validation-Mean-LP checkpoint (held-out subjects, never test), not the final epoch. To show the budget does not drive the result, we linear-probed AXON and Dense at epochs 5--50 (Mean LP over the six tasks, single seed; Table~\ref{tab:budget}). Validation-best was epoch 10 (AXON) and epoch 9 (Dense); at the final epoch AXON still leads 0.568 vs.\ 0.543. The ranking is stable at every budget: AXON leads by $+$0.025--0.028 Mean LP and Dense never catches up, so the gain is not faster learning that more compute would erase. We do not claim it holds for unlimited training, only that it is stable across every budget we tested.

\begin{table}[h]
\centering\small
\caption{Mean LP over the six tasks at matched pretraining budgets (single seed). AXON leads at every epoch.}
\label{tab:budget}
\begin{tabular}{lcccccc}
\toprule
Epoch & 5 & 10 & 20 & 30 & 40 & 50 \\
\midrule
AXON  & 0.584 & 0.579 & 0.580 & 0.574 & 0.570 & 0.568 \\
Dense & 0.556 & 0.551 & 0.552 & 0.549 & 0.545 & 0.543 \\
\bottomrule
\end{tabular}
\end{table}

\section{Why Two Layers Connect Every Pair of Tokens}
\label{app:theory}

\subsection{Axis graph diameter}
\label{app:axis_diameter}

AXON removes the dense all-to-all path. In one layer a token attends only to tokens on its own electrode (the temporal path) and to tokens at its own time step (the spatial path). A concern is that this cuts the model off from the rest of the recording: a token on electrode $c$ at time $t$ never sees electrode $c'$ at time $t'$ directly. The proposition below shows that this is not so. On a full channel-time grid, any two tokens are connected after two layers, so the model keeps its global reach. What changes is which pairs interact directly within one layer, and how many. This is why we describe the design as changing which pairs interact directly, not the model's reach, and it is the fact behind the statement in Section~\ref{ssec:factorized} that any two tokens can still meet after two layers.

\begin{proposition}[Axis graph has diameter at most two]
\label{prop:axis_diameter}
On the full channel-time grid $\Omega=\mathcal{C}\times\mathcal{T}$, the graph with edges between tokens that share either the same channel or the same time index has diameter at most two. After two stacked axis-attention layers, any token can receive information from any other token. The number of possible one-hop attention interactions is
\[
|E_{\rm axis}|\le CT^2+TC^2 = CT(C+T),
\]
versus $|E_{\rm dense}|=(CT)^2$ for the dense graph.
\end{proposition}

\begin{proof}
Take any two tokens $u=(c,t)$ and $v=(c',t')$. If $c=c'$ or $t=t'$, then $u$ and $v$ are connected by one axis edge. Otherwise, $u$ is connected to $(c',t)$ by a spatial edge (same time step), and $(c',t)$ is connected to $v=(c',t')$ by a temporal edge (same channel). Every pair is thus connected by a path of length at most two. The edge-count bound follows from $C$ temporal groups of size $T$ and $T$ spatial groups of size $C$.
\end{proof}

\section{Ablation Summary}
\label{app:ablation_full}

\subsection{Interpreted ablation summary}

\begin{table}[h]
\centering
\caption{EEG ablation summary. Values are unweighted mean balanced accuracy across six downstream tasks. Diagnostic variants are included to show task-specific tradeoffs; bold indicates the best mean in each column among evaluated variants.}
\label{tab:ablation}
\resizebox{\linewidth}{!}{
\begin{tabular}{lccp{7.5cm}}
\toprule
Variant & Mean LP & Mean FT & Purpose of comparison \\
\midrule
Dense & 0.550 & 0.606 & Dense attention baseline with split positional encoding. \\
Dense-L & 0.555 & 0.633 & Parameter-matched dense control. \\
Divided-ST & 0.560 & 0.628 & Sequential divided space-time operator~\citep{bertasius2021space} in the identical setup (Appendix~\ref{app:divided_st}). \\
Dense + joint PE & 0.540 & 0.600 & Tests whether joint positional encoding improves over split PE. \\
Dense + 2D-RoPE & 0.554 & --- & Axial 2D rotary position embedding on the dense graph; single seed vs.\ a single-seed dense reference (0.546). \\
AXON-Fixed & 0.558 & 0.640 & Tests axis factorization without token-dependent routing. \\
Fixed gate, schedule init & 0.555 & 0.636 & Tests whether initializing a fixed gate to AXON's learned layer schedule is sufficient. \\
AXON-TokenGated & 0.578 & 0.643 & Tests token-gated temporal/spatial factorization without multiscale temporal windows. \\
\textbf{AXON} & \textbf{0.579} & \textbf{0.656} & Final no-global model with token axis gate and two-scale temporal branch. \\
AXON without stop-gradient & 0.577 & --- & Gate reads the live representation instead of a detached copy; single seed. \\
\midrule
AXON-withGlobal & 0.555 & 0.619 & Tests whether adding a dense global branch improves the factorized encoder. \\
KNN spatial constraint & 0.510 & 0.609 & Tests whether local electrode-neighbour masking helps the spatial path. \\
Head-split multiscale & 0.553 & 0.628 & Tests an alternative multiscale temporal implementation. \\
Temporal decay & 0.556 & 0.642 & Diagnostic locality bias; helps some event-like tasks but reduces mean performance. \\
AXON-Decay continuation & 0.579 & 0.633 & Diagnostic continuation run; maintains LP but reduces FT. \\
\bottomrule
\end{tabular}
}
\end{table}

\paragraph{Reading the ablations.}
The ablation summary supports three conclusions. First, axis factorization improves over dense attention even after controlling for parameter count (Dense-L). Second, token gating provides the largest additional LP gain over fixed factorization ($+$0.020), while the two-scale temporal path contributes a smaller task-selective refinement. Third, the negative variants show that adding a dense global branch, hard spatial masks, or extra temporal-scale machinery does not improve mean transfer. AXON is the only variant that achieves both the best mean LP and the best mean FT; variants that improve selected tasks (e.g., AXON-Decay) do not improve the mean FT objective and are therefore treated as diagnostic, task-selective extensions. The token gate's LP advantage over AXON-Fixed arises from training-time gradient diversity rather than inference-time routing; see Appendix~\ref{app:schedule_init} for the schedule-initialized fixed gate experiment that isolates this effect.

\paragraph{Three windows.}
We also tried three windows (two, five, and all patches) with an entropy penalty that pushes the scale mix to use all three. Mean LP fell to 0.558. The reconstruction loss barely distinguishes the three windows, so the penalty dominates and the mix stays close to uniform; two windows with a long-window start was the best we found.

\paragraph{Temporal locality (AXON-Decay).}\label{app:temporal_locality_full}
AXON-Decay adds a learnable temporal-distance bias to the full-window temporal path, so that nearby time steps get more weight. It helps tasks driven by short events, motor imagery (LP 0.455 $\to$ 0.487) and seizure detection (siena LP 0.867 $\to$ 0.894); it matches AXON on Mean LP (0.579), lowers Mean FT (0.656 $\to$ 0.633), and hurts workload, a sustained-state task. So the best temporal context length depends on the task, and we treat AXON-Decay as a diagnostic, not a replacement for AXON.

\paragraph{2D RoPE control.}
We ran axial 2D RoPE~\citep{su2024roformer,heo2024rope}: each head (dim 64) is split so attention depends only on the relative $(\Delta t, \Delta c)$ offset while the graph stays fully dense. Result (same pipeline as Table~\ref{tab:main_results}; single seed, vs.\ a single-seed dense reference): Dense 0.546 $\to$ Dense+2D-RoPE 0.554, a $+$0.008 Mean LP gain concentrated almost entirely on adftd (our highest-variance task, $\pm$0.034 across seeds). Axis factorization gives $+$0.029, ${\approx}$3.6$\times$ the RoPE delta. Thus 2D RoPE does not substitute for factorization (though it does not fail either). This is expected: RoPE changes how position enters the scores but leaves the graph dense; every off-axis pair stays available, whereas AXON removes those edges by construction. After pretraining the dense encoder still places 30--63\% of its attention on off-axis pairs (token pairs sharing neither the same electrode nor the same time step; Table~\ref{tab:relation_mass}, Figure~\ref{fig:relation_mass}), so it does not suppress those interactions on its own, and a relative-position code gives it no mechanism to. The two are orthogonal, not substitutes.

\begin{table}[!t]
\centering
\caption{Architectural variants summary. All models use the same pretraining corpus and schedule; they differ only in attention structure.}
\label{tab:variants}
\small
\begin{tabular}{lcccc}
\toprule
Variant & Temporal windows & Global path & Key change & Mean LP \\
\midrule
Dense & W=$-$1 & Full & Reference & 0.550 \\
Divided-ST & W=$-$1 & No & Sequential T$\to$S blocks & 0.560 \\
AXON-Fixed & W=$-$1 & No & Fixed factorized axes & 0.558 \\
AXON-TokenGated & W=$-$1 & No & Per-token axis gating & 0.578 \\
AXON-withGlobal & W=[5,$-$1] & Yes & +Global dense path & 0.555 \\
\textbf{AXON} & \textbf{[5,$-$1]} & \textbf{No} & +Two-scale gate & \textbf{0.579} \\
AXON-Decay & {[5,$-$1]} + decay & No & +Temporal locality & 0.579 \\
\midrule
Head-split multiscale & [5,$-$1] static & No & Static scale split & 0.553 \\
3-scale gate & [2,5,$-$1] & No & 3-scale gate & 0.558 \\
\bottomrule
\end{tabular}
\end{table}

\subsection{Divided space-time baseline}
\label{app:divided_st}
The divided space-time baseline is pretrained in our exact EEG setup: same pooled corpus, same 10-20 montage and tokenisation, batch size 4096, peak LR $2.4\times10^{-4}$, fused AdamW, bfloat16, the identical pretraining budget, the identical MAE objective (L1 on masked patches plus the pooled-attention auxiliary loss), and the same split positional encoding. Only the attention operator differs. We implement the divided space-time block of \citet{bertasius2021space} faithfully: attention along time, then attention along channels, each with its own residual connection and with the temporal output projection, in place of AXON's parallel token-gated axis mixture. Parameter count and per-layer compute are comparable to AXON's ${\sim}$141.9M. The two models converge to matched reconstruction loss, so neither is under-trained relative to the other. Metric is balanced accuracy on subject-disjoint splits identical across models; per-task results are in Table~\ref{tab:main_results}, where $\pm$ is the standard deviation across the 3 downstream seeds.

Divided-ST recovers part of the gap to AXON, so restricting attention to the two axes helps on its own. The rest of the gap is what the parallel, gated composition adds. In Divided-ST the axis order is fixed and identical for every token and every layer; in AXON the gate sets the temporal/spatial balance per token and per layer, and Appendix~\ref{app:gate_mechanism} shows that this learned per-layer balance is where most of the gate's benefit lies.

\section{Analysis of the Global-Path Variant}
\label{app:global_shortcut}
Section~\ref{ssec:global} states our hypothesis: under masked pretraining the dense global path is a shortcut, averaging over visible tokens to guess a masked patch, so the encoder builds fewer axis-specific features. Here we give the evidence behind that reading.

\paragraph{Same reconstruction loss, different features.}
AXON and AXON-withGlobal reach the same reconstruction loss (Figure~\ref{fig:recon_loss}), yet AXON-withGlobal transfers worse (Mean LP 0.555 vs.\ 0.579, Mean FT 0.619 vs.\ 0.656). So the difference is in the features the two encoders build, not in how well they reconstruct.

\begin{figure}[h]
\centering
\includegraphics[width=0.75\linewidth]{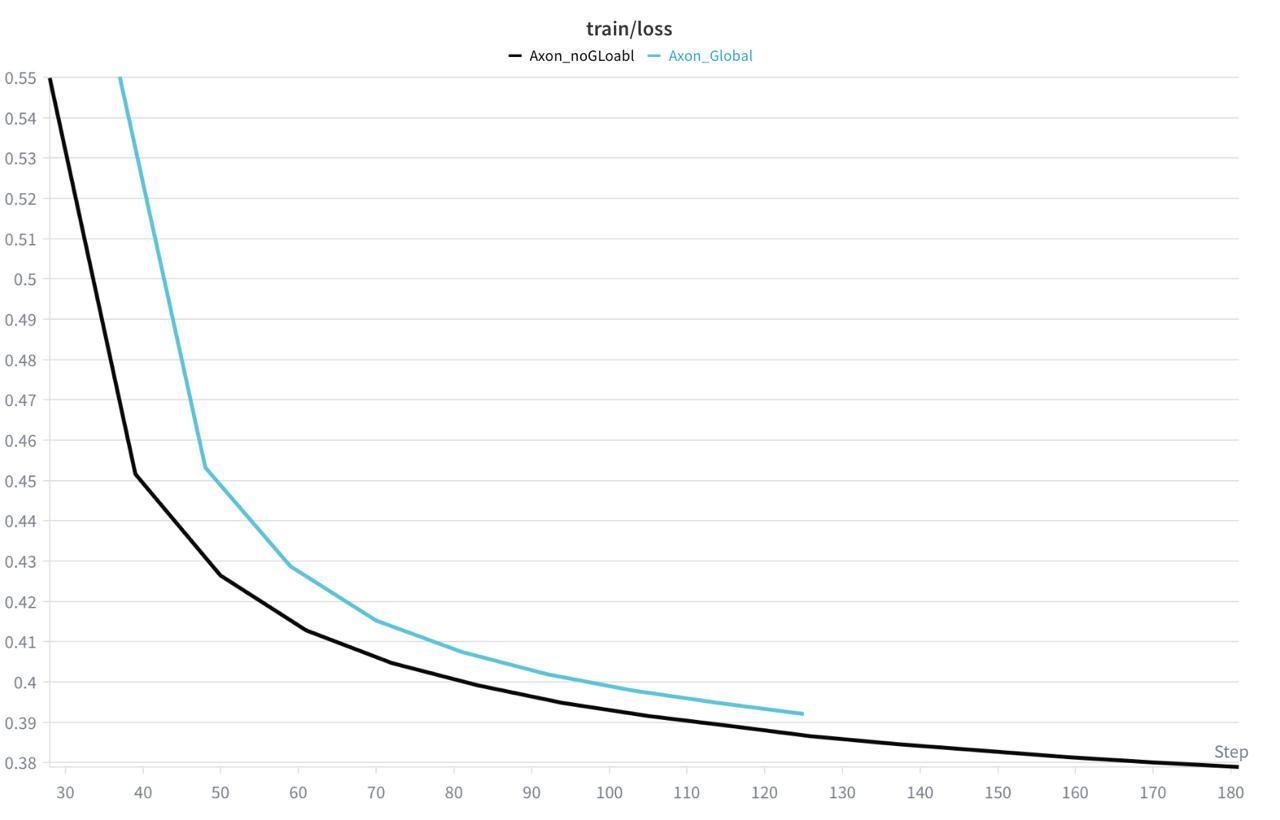}
\caption{MAE reconstruction loss during pretraining. AXON (no global path) and AXON-withGlobal converge to comparable reconstruction loss ($\approx$0.38--0.39), yet AXON-withGlobal underperforms on mean LP (0.555 vs.\ 0.579) and mean FT (0.619 vs.\ 0.656). The global path does not improve reconstruction; it changes how the encoder reconstructs, routing mass through dense averaging ($\gamma\!=\!0.724$ at layer~1) rather than through axis-structured features.}
\label{fig:recon_loss}
\end{figure}

\paragraph{Effective rank.}
Effective rank~\citep{roy2007effective} counts how many dimensions a set of features actually uses; a higher value means richer, less collapsed features. We compute it on the final mean-pooled embeddings of up to 512 validation samples per dataset (Table~\ref{tab:effective_rank}). Removing the global path raises the effective rank on five of the six datasets. The largest jump is on motor imagery, from 19.97 to 29.05, and motor imagery is also the task where AXON gains most over Dense (Table~\ref{tab:main_results}). This fits the shortcut reading: motor imagery depends on fine, local spatio-temporal structure that averaging over all tokens washes out.

\begin{table}[h]
\centering
\caption{Effective rank and representation similarity (Linear CKA) of final mean-pooled embeddings between the factorized noGlobal model and the withGlobal model.}
\label{tab:effective_rank}
\resizebox{0.7\linewidth}{!}{
\begin{tabular}{lccc}
\toprule
Dataset & \multicolumn{2}{c}{Effective Rank $\uparrow$} & CKA Similarity $\downarrow$ \\
\cmidrule(lr){2-3}
& \emph{withGlobal} & \emph{noGlobal} & \emph{withGlobal} vs \emph{noGlobal} \\
\midrule
adftd & 17.76 & \textbf{21.29} & 0.9748 \\
bcic\_2a & 66.08 & \textbf{77.66} & 0.9363 \\
hmc & 22.83 & \textbf{28.79} & 0.8662 \\
motor & 19.97 & \textbf{29.05} & 0.8412 \\
siena & \textbf{28.93} & 27.19 & 0.9634 \\
workload & 51.58 & \textbf{65.48} & 0.8698 \\
\bottomrule
\end{tabular}
}
\end{table}

\paragraph{CKA similarity.}
Linear CKA~\citep{kornblith2019cka} measures how similar two sets of features are (1.0 identical, 0.0 unrelated). The features of AXON and AXON-withGlobal are least similar on motor imagery (0.84), sleep staging (0.86) and workload (0.87), and almost identical on seizure detection and dementia (above 0.96). The global path changes the features most on the tasks with the most temporal or spatial structure, and least where the two models also perform alike. Per-layer curves for all six datasets are in Figure~\ref{fig:cka_layer_grid}.

\begin{figure}[h]
\centering
\begin{subfigure}[t]{0.48\linewidth}
\centering
\includegraphics[width=0.95\linewidth]{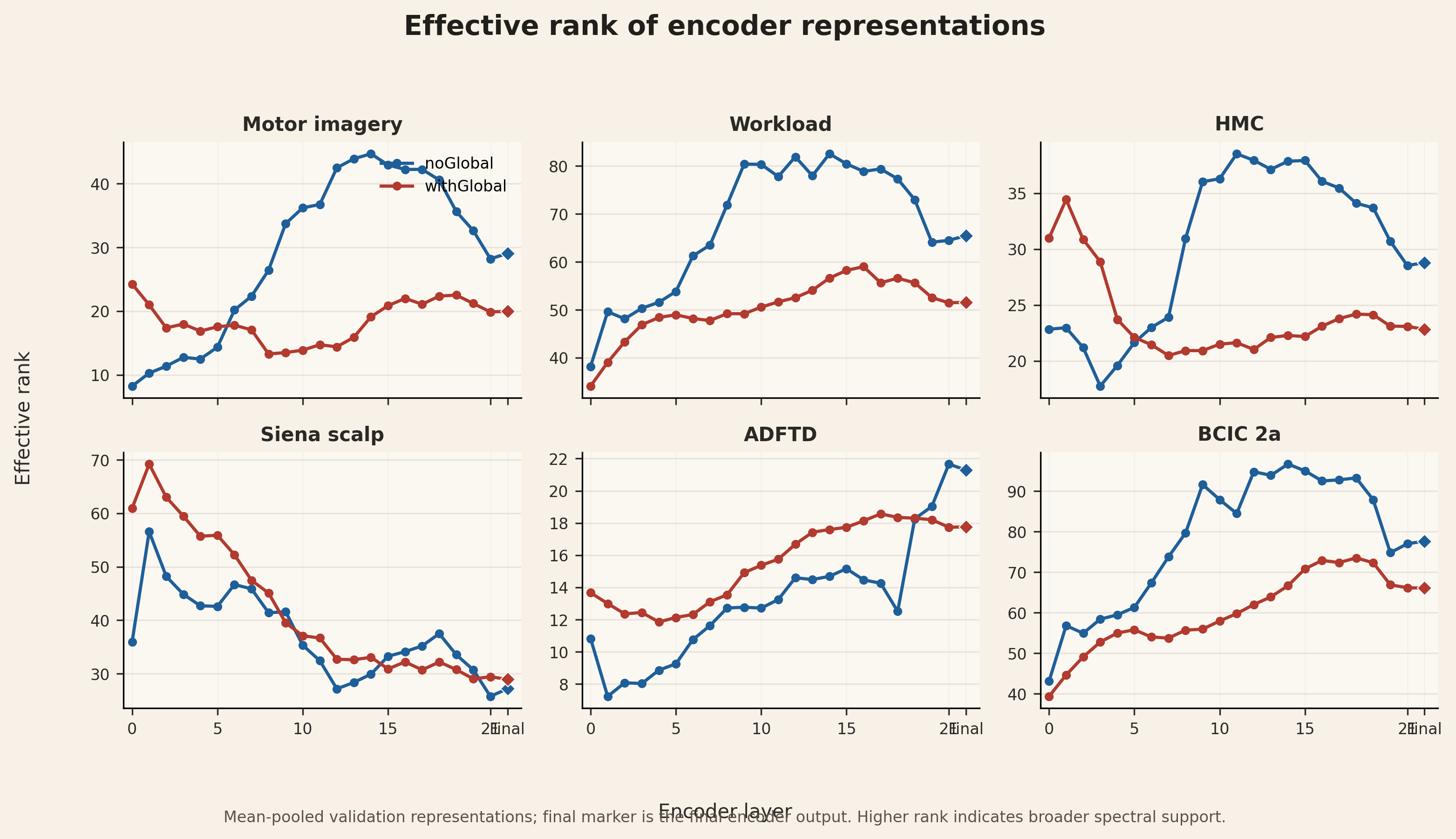}
\caption{Layer-wise effective rank, all six datasets.}
\end{subfigure}\hfill
\begin{subfigure}[t]{0.48\linewidth}
\centering
\includegraphics[width=0.95\linewidth]{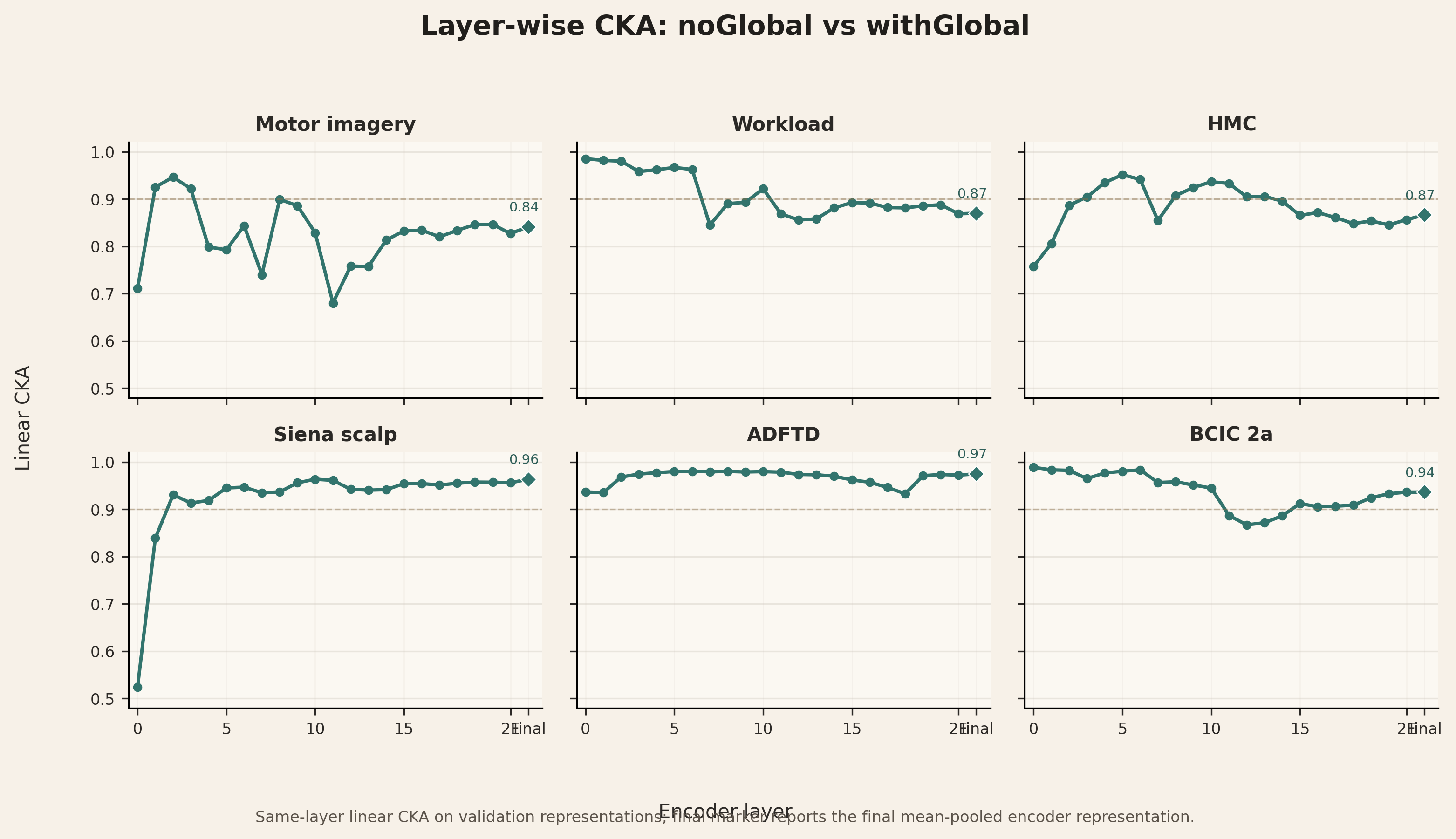}
\caption{Layer-wise Linear CKA, all six datasets.}
\end{subfigure}
\caption{Layer-wise effective rank and CKA divergence between factorized and dense-global representations.}
\label{fig:cka_layer_grid}
\end{figure}

\section{Gate Mechanism Analysis}
\label{app:gate_mechanism}

\subsection{Understanding the gate}
\label{app:gate_profile}
Figure~\ref{fig:gate_profile} shows the average gate weight per layer over the six downstream datasets. The gate is not a uniform mixer. The first layer leans on the temporal path ($\alpha = 0.745$), layers 4--7 lean on the spatial path ($\beta = 0.686$--$0.707$), and late layers return toward the temporal path ($\alpha = 0.651$ at layer 19). So the gate learns a different temporal/spatial balance at each depth. The rest of this appendix asks whether the model actually depends on these values.

\begin{figure}[h]
\centering
\includegraphics[width=0.85\linewidth]{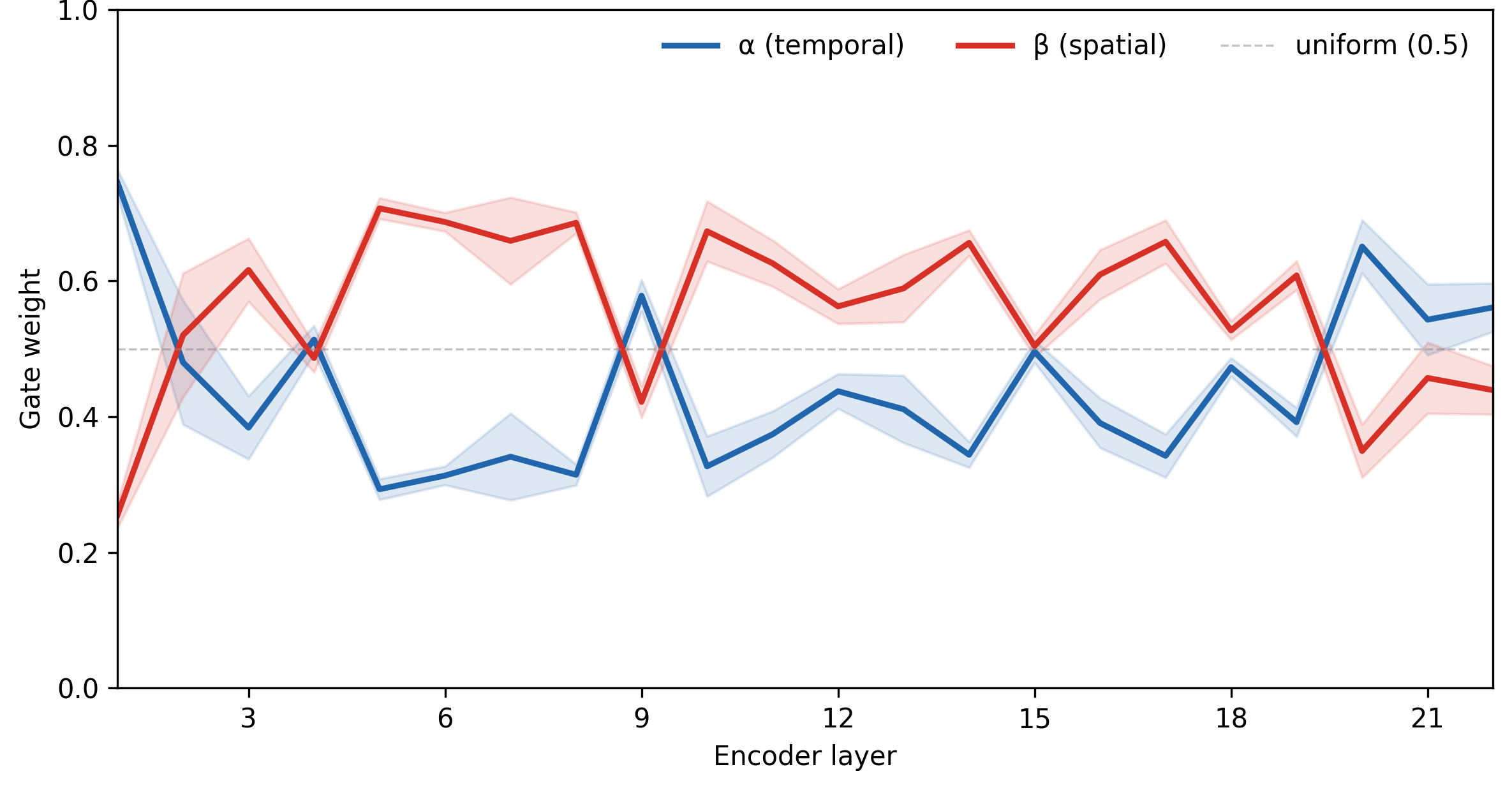}
\caption{AXON axis gate across the 22 encoder layers: gate weights averaged over all six downstream datasets (mean $\pm\,\sigma$), on validation batches with the frozen encoder. Temporal-heavy in the first layer, spatial-heavy in the middle layers, and back toward temporal in the late layers.}
\label{fig:gate_profile}
\end{figure}

\subsection{Gate intervention diagnostics}
\label{app:gate_interventions}

We override the axis gate while keeping the pretrained encoder frozen, then run linear probe evaluation (full dataset splits). Each intervention replaces the learned per-token gate $[\alpha_i, \beta_i]$ with a modified version that removes a specific component of the routing, isolating what the gate actually contributes. We group the seven interventions by the question they answer.

\paragraph{Is axis factorization itself necessary?}
\begin{itemize}[leftmargin=*,noitemsep]
\item \textbf{temporal\_only}: Force $[\alpha_i,\beta_i]=[1,0]$ for all tokens. Only the temporal path contributes.
\item \textbf{spatial\_only}: Force $[\alpha_i,\beta_i]=[0,1]$. Only the spatial path contributes.
\end{itemize}

\paragraph{Does soft mixing matter, or can the model commit to one axis?}
\begin{itemize}[leftmargin=*,noitemsep]
\item \textbf{hard\_argmax}: Take the learned gate, find the dominant axis, and set it to 1.0 with the other at 0.0. E.g., $[0.6, 0.4] \to [1.0, 0.0]$.
\item \textbf{uniform}: Force $[\alpha_i,\beta_i]=[0.5,0.5]$ for all tokens. No routing at all equal weight to both paths.
\end{itemize}

\paragraph{Is the gate's value per-token content routing, or a depth/position schedule?}
\begin{itemize}[leftmargin=*,noitemsep]
\item \textbf{layer\_mean}: Calibrate over 50 forward passes to compute the average gate per layer, averaging across all tokens and batches. At inference, every token in layer $\ell$ receives the layer-$\ell$ mean gate, regardless of content. Tests whether the depth schedule alone is sufficient.
\item \textbf{position\_mean}: Calibrate the average gate per (layer, channel, time-patch) position. At inference, a token at position $(c,t)$ in layer~$\ell$ receives the calibrated mean for that position, regardless of signal content. Tests whether position-aware routing adds value beyond the depth schedule.
\item \textbf{shuffled}: Run the gate MLP normally to produce $[\alpha_i,\beta_i]$ for each token, then randomly permute the gate values within each layer. The marginal distribution of gate values per layer is exactly preserved, but the token$\leftrightarrow$gate correspondence is destroyed. Tests whether it matters which token gets which gate value.
\end{itemize}

Table~\ref{tab:gate_interventions} reports summary statistics; Figure~\ref{fig:gate_heatmap} shows per-dataset results.

\begin{table}[h]
\centering
\caption{Gate intervention diagnostics under frozen-encoder LP. Mean is over all 6 downstream datasets. Scores are balanced accuracy on the test subjects at the validation-selected epoch, averaged over 3 downstream seeds for five datasets; siena uses a single seed and the class-balanced training loader.}
\label{tab:gate_interventions}
\small
\begin{tabular}{lcc}
\toprule
Gate mode & Mean BAC & $\Delta$ vs learned \\
\midrule
Learned (reference)        & 0.571 & --- \\
Uniform $(1/2,1/2)$       & 0.539 & $-5.6\%$ \\
Layer mean                 & 0.559 & $-2.1\%$ \\
Position mean    & 0.567 & $-0.8\%$ \\
Shuffled                   & 0.564 & $-1.3\%$ \\
Hard argmax                & 0.501 & $-12.2\%$ \\
Temporal only              & 0.471 & $-17.5\%$ \\
Spatial only               & 0.446 & $-22.0\%$ \\
\bottomrule
\end{tabular}
\end{table}

\begin{figure}[h]
\centering
\includegraphics[width=0.95\linewidth]{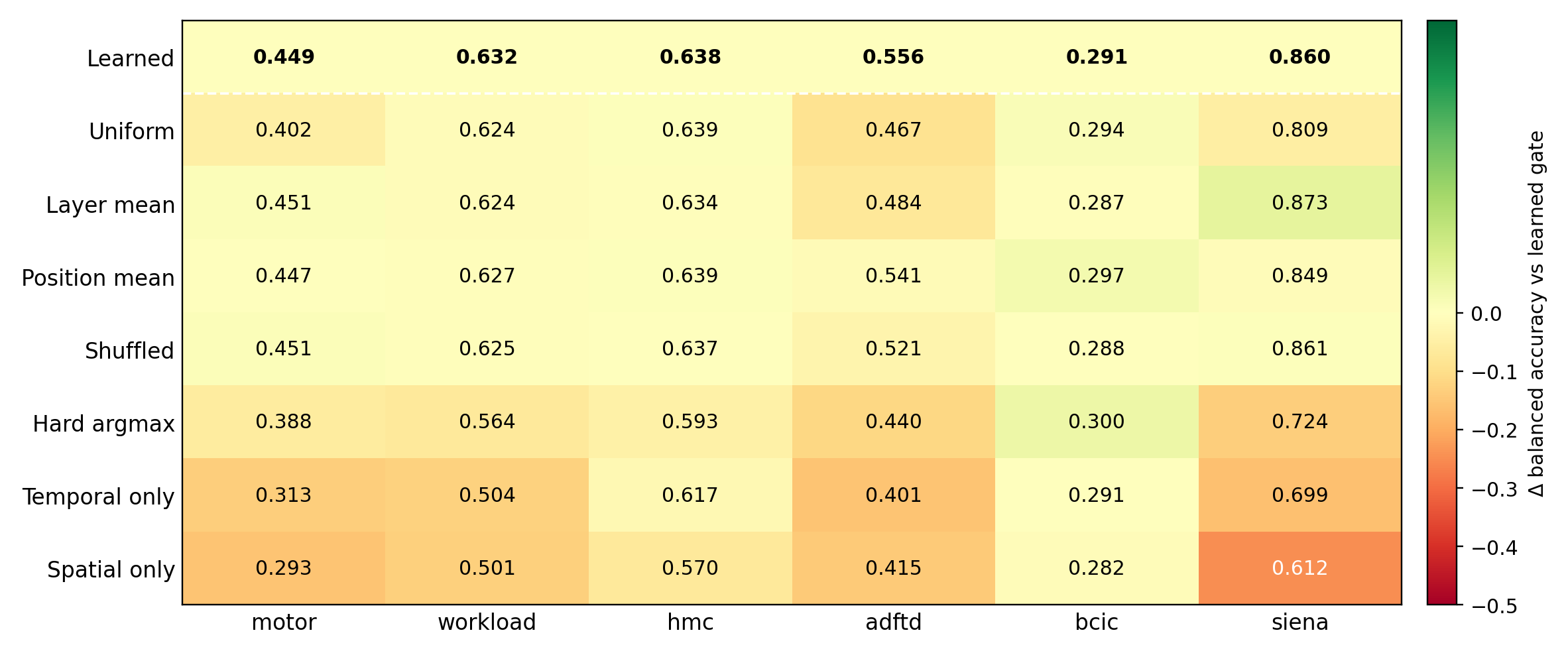}
\caption{Per-dataset gate intervention heatmap (frozen-encoder LP, 6 datasets; same protocol as Table~\ref{tab:gate_interventions}). Cell colour is the change from the learned gate. The learned gate outperforms uniform and hard-argmax routing on most datasets, but layer-mean, position-mean, and shuffled gates stay close to learned, indicating that exact token-gate alignment is not the dominant effect.}
\label{fig:gate_heatmap}
\end{figure}

\smallskip
\footnotesize
\textit{Protocol note.} The learned-gate reference is 0.571 in this intervention setting, compared with the main AXON LP score of 0.579 in Table~\ref{tab:main_results}. All intervention results are measured relative to this within-table learned reference. 
\normalsize

\subsection{Token diversity and content dependence}

We compute two scalar metrics to characterise within-layer gate variation. $D_\mathrm{token}$ measures how different individual token gates are from the layer mean; $D_\mathrm{content}$ subtracts out fixed channel/time position effects, isolating variation that depends on the signal content of the current input. $D_\mathrm{content}$ peaks sharply at layers 9 and 12 (Figure~\ref{fig:d_content}), indicating that signal-driven routing is concentrated at mid-depth rather than distributed uniformly across the encoder. However, the gate intervention results (Table~\ref{tab:gate_interventions}) show that this token-level variation is not the dominant source of downstream gain: shuffled gates drop only $-$1.3\% vs.\ learned.

\begin{figure}[h]
\centering
\includegraphics[width=0.85\linewidth]{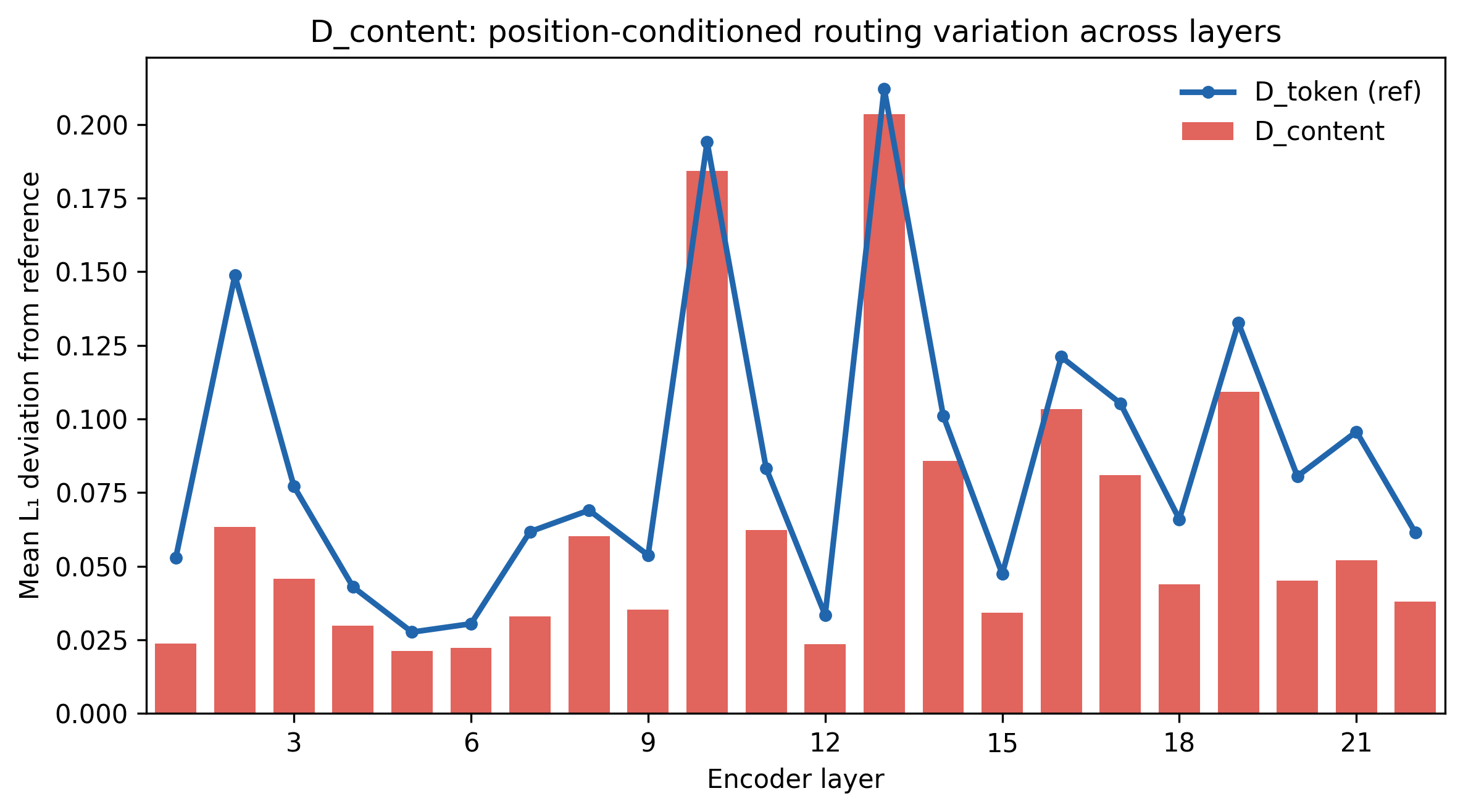}
\caption{$D_\mathrm{token}$ and $D_\mathrm{content}$ across 22 encoder layers. Peaks at layers 9 and 12 show that some layers exhibit genuine token-level gate diversity. However, the intervention results above show that this variation is not the dominant source of downstream gain: shuffled gates (which preserve the gate distribution but destroy token-gate alignment) drop only $-$1.3\% vs learned.}
\label{fig:d_content}
\end{figure}

\subsection{Analysis of the AXON-TokenGated vs AXON-Fixed gap}
\label{app:schedule_init}
Giving every token in a layer that layer's mean gate costs only $-2.1\%$ (Section~\ref{app:gate_interventions}). So at test time, one mixing weight per layer is almost enough. But AXON-Fixed learns exactly one weight per layer, and it reaches only 0.558 Mean LP, against 0.578 for AXON-TokenGated. Why?

One guess is initialisation: AXON-Fixed starts from a neutral mixture and may never find the right weight for each layer. We tested this. We trained AXON-Fixed again, this time starting each layer's weight at the value the AXON gate had learned for that layer (Section~\ref{app:gate_profile}). Nothing else changed. Table~\ref{tab:stage1b} shows the result: 0.555 Mean LP, the same as before (0.558), still far below 0.578. The right starting point does not help. So initialisation is not the reason.

What is left is how the two models train. With a per-token gate, tokens in the same layer get different mixtures during training, so the temporal and spatial paths are trained on more varied signals. At the end of training the exact per-token values no longer matter much (shuffling them costs only $-1.3\%$), but the paths they trained are better. Dropout works the same way: it does nothing at test time, but it changes the weights that training ends with.

\begin{table}[h]
\centering
\caption{AXON-Fixed re-trained with each layer's weight initialised to AXON's learned per-layer balance. For comparison: AXON-Fixed with neutral initialisation reaches 0.558 Mean LP, AXON-TokenGated 0.578.}
\label{tab:stage1b}
\small
\begin{tabular}{lcc}
\toprule
Dataset & LP BAC & FT BAC \\
\midrule
motor & 0.354 & 0.598 \\
workload & 0.585 & 0.714 \\
hmc & 0.673 & 0.732 \\
siena & 0.878 & 0.873 \\
adftd & 0.557 & 0.588 \\
bcic & 0.283 & 0.311 \\
\midrule
\textbf{Mean} & \textbf{0.555} & \textbf{0.636} \\
\bottomrule
\end{tabular}
\end{table}

\subsection{Temporal scale gate routing}

Figure~\ref{fig:scale_gate} shows the per-layer routing between the short-window ($W_s\!=\!5$) and full-window ($W_l\!=\!-1$) temporal branches across all 22 encoder layers. The scale gate strongly favours the long-window branch ($\approx$79\% routing mass), consistent with the initialisation bias toward $W_l$. A small subset of layers routes appreciable mass to the short window, suggesting that local transient structure is selectively useful at those depths.

\begin{figure}[h]
\centering
\includegraphics[width=0.85\linewidth]{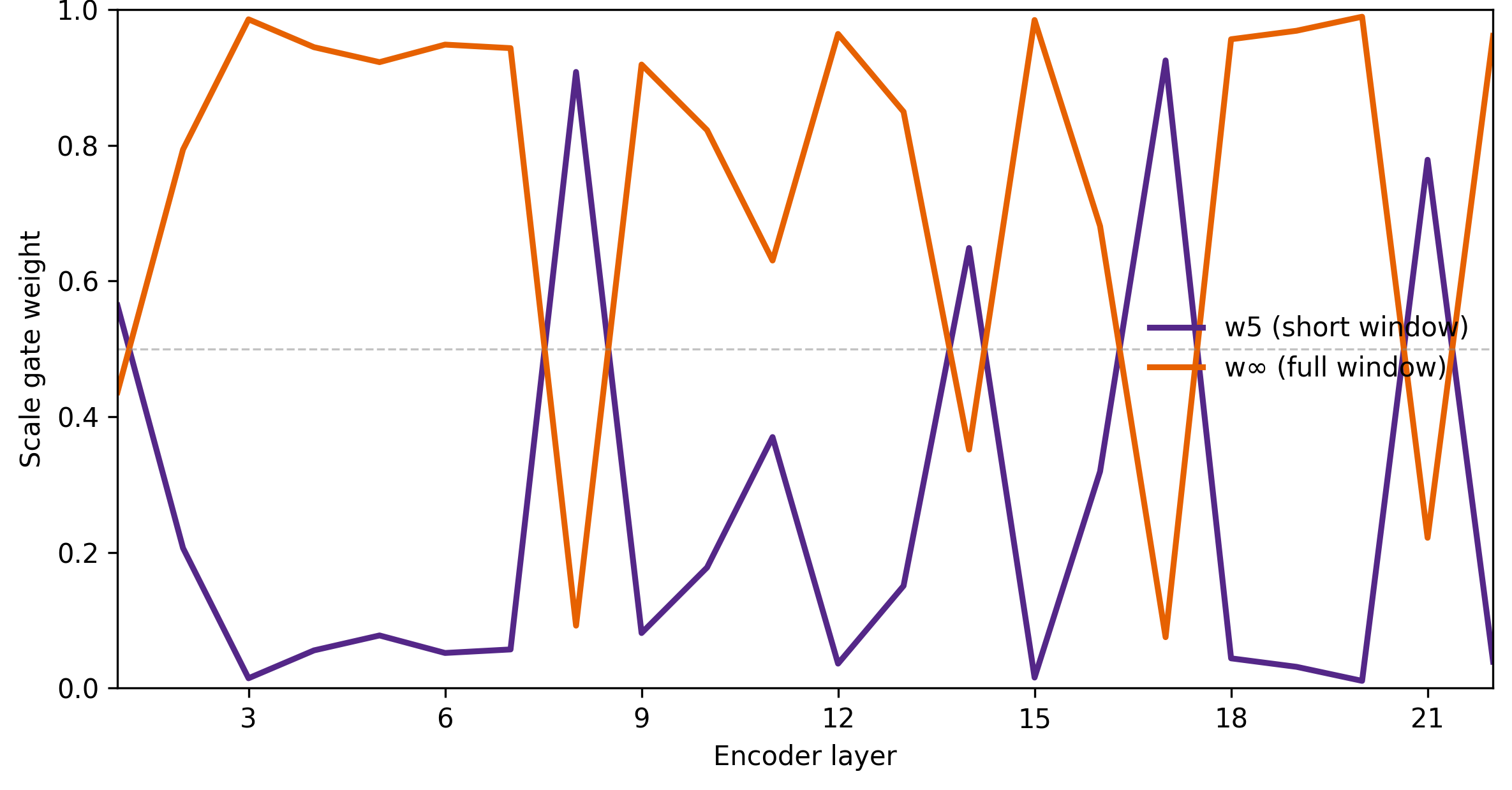}
\caption{Scale gate routing: short window ($W\!=\!5$) vs.\ full window ($W\!=\!\infty$) across 22 layers. The model usually favours broad temporal context, but uses short temporal windows in a few selected layers. The two-scale temporal design contributes $+$0.001 Mean LP and $+$0.013 Mean FT over the single-scale variant; the scale gate is a secondary improvement.}
\label{fig:scale_gate}
\end{figure}

\section{Interpreting the Attention}
\label{app:electrodes}

\subsection{Does the trained dense model use the off-axis pairs that AXON removes?}
\label{ssec:isotropic}
For a query token in the modal $C=21$, $T=11$ grid, 200 of the 231 tokens share neither its channel nor its time step. We call these off-axis pairs. An untrained dense model spreads its attention uniformly, so about 87\% of its attention starts on off-axis pairs. The question is how much of that the dense model learns to remove during pretraining. We measured it on the motor imagery grid ($C=64$, $T=4$), because motor imagery is the task with AXON's largest gain and the 64-channel layout makes the pair types easy to separate; there the uniform off-axis share is 73.8\%. After pretraining, attention within the same time step rises to 62.4\% and the off-axis share falls to 30.9\% on average (Table~\ref{tab:relation_mass}). But it never goes away: the first layer still places 63\% of its attention off-axis, almost the untrained value, and the last layer drifts back to 46\% (Figure~\ref{fig:relation_mass}). The dense model learns the axis structure only partly and unevenly across depth. AXON assigns zero off-axis attention within a layer by construction.

\begin{table}[h]
\centering\small
\caption{Relation-type attention mass (\%) in the dense baseline on the default motor\_mv\_img downstream grid ($C\!=\!64,T\!=\!4$), averaged across all 22 layers. After MAE pretraining, spatial mass rises and off-axis mass drops, but residual off-axis routing remains. AXON assigns zero one-layer off-axis mass by construction.}
\label{tab:relation_mass}
\begin{tabular}{lcccc}
\toprule
Condition & Self & Temporal & Spatial & Off-axis \\
\midrule
Uniform reference & 0.39 & 1.17 & 24.6 & 73.8 \\
Dense (trained) & 4.65 & 2.06 & 62.4 & 30.9 \\
AXON (by construction) & --- & 100 (temporal path) & 100 (spatial path) & 0 \\
\bottomrule
\end{tabular}
\end{table}

\begin{figure}[h]
\centering
\includegraphics[width=0.95\linewidth]{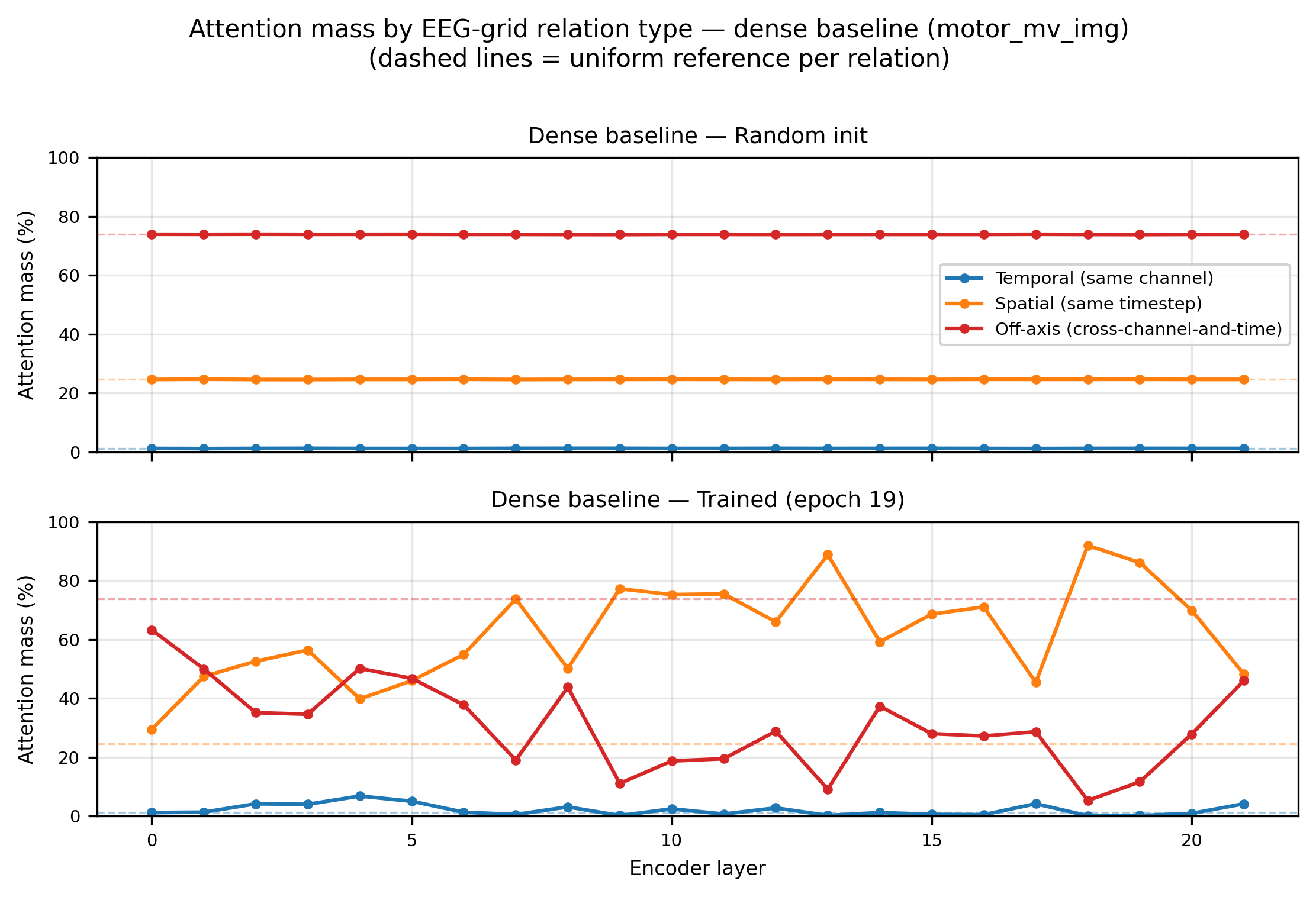}
\caption{Relation-type attention mass per layer in the dense baseline ($C\!=\!64, T\!=\!4$). \textbf{Top}: Random initialisation matches the uniform complete-graph prior (dashed lines). \textbf{Bottom}: After MAE pretraining, dense attention discovers spatial dominance in mid-depth layers but fails to fully suppress off-axis interactions. Early and late layers retain up to 63\% off-axis mass, leaving off-axis interactions that AXON removes by construction.}
\label{fig:relation_mass}
\end{figure}

\subsection{Does AXON rely on the electrodes that physiology predicts?}
\label{app:covering}
The task is four-class motor imagery (motor\_mv\_img, Table~\ref{tab:datasets}). In each trial the subject imagines moving the left hand, the right hand, both fists, or the feet. The body is controlled from the opposite side of the brain: the left hand from the right hemisphere, the right hand from the left. Both fists use both sides. The feet are controlled from the midline. So for each class we know which electrodes a good classifier should be using.

We take the trained AXON motor-imagery classifier and its held-out test subjects. We cover the seven electrodes over the left motor cortex, so the model gets no signal from them, and measure how much each class's recall changes. Recall is the fraction of a class's trials that the model labels correctly. Then we do the same for the seven electrodes over the right motor cortex. We chose the electrodes and the measure before looking at any result.

If the classifier uses the correct electrodes, covering one side should mainly hurt the opposite hand, and should not hurt both fists or feet in a one-sided way. If it used all electrodes alike, covering either side would hurt all four classes alike. Table~\ref{tab:covering} shows the first pattern. Covering the left side drops right-hand recall by 0.101 and does not hurt the left hand ($+$0.058). Covering the right side drops left-hand recall by 0.127 and does not hurt the right hand ($-$0.004). Both fists and feet show no one-sided change. So the drop is not just ``fewer electrodes, worse accuracy'': each hand's decision depends on the electrodes over the opposite hemisphere, exactly where physiology says it should. We use this test rather than an attention map because it changes the input and watches the decision; an attention map only shows where the weights point.

\begin{table}[h]
\centering\small
\caption{Change in recall after covering the left or right motor strip (motor imagery, held-out subjects); negative means worse.}
\label{tab:covering}
\begin{tabular}{lrr}
\toprule
Imagined movement & Cover LEFT strip & Cover RIGHT strip \\
\midrule
Left hand  & $+$0.058 & $\mathbf{-0.127}$ \\
Right hand & $\mathbf{-0.101}$ & $-$0.004 \\
Both fists & $-$0.027 & $-$0.001 \\
Feet       & $+$0.013 & $+$0.052 \\
\bottomrule
\end{tabular}
\end{table}

\subsection{Does the accuracy depend on particular electrodes?}
\label{app:deletion}
Deleting randomly chosen electrodes at test time, AXON stays ahead of dense (balanced accuracy) at every level, from 0.377 vs.\ 0.321 intact to 0.274 vs.\ 0.263 with 94\% of electrodes removed. Here the classification head is trained on mean-pooled frozen features rather than through our full evaluation harness, so absolutes differ from Table~\ref{tab:main_results} and only the model-to-model comparison is meaningful. This matters clinically, where reduced montages and failed electrodes are routine.

\section{Cross-Modal Generalization (AudioMAE)}
\label{app:audio}

\paragraph{Gap to published AudioMAE.}
The absolute mAP values are not directly comparable to those reported by \citet{huang2022audiomae}. For example, our best full AudioSet-2M FT result is 14.29 mAP, whereas published AudioMAE-style results report much higher absolute mAP under a substantially different training recipe (e.g., 37.0 on AudioSet-20K). This gap reflects five controlled differences: (1)~spectrogram resolution (128 vs.\ 8 frequency bins, a 16$\times$ reduction that brings the frequency axis to a size comparable to the EEG electrode axis); (2)~pretraining compute (4 epochs vs.\ 32 epochs, $\sim$9$\times$ fewer sample-views); (3)~model capacity ($d\!=\!512$ vs.\ $d\!=\!768$, $\sim$50\% fewer parameters); (4)~decoder depth (4 vs.\ 16 layers); and (5)~FT augmentation (no Mixup, SpecAugment, or DropPath). Crucially, the factorized and dense variants share all five of these constraints, so the within-setup comparison is valid.

\subsection{Small-scale preliminary results (18K AudioSet clips)}

Before scaling to 200K clips, we verified factorized attention at 1\% of AudioSet ($\sim$18K clips). The same four-variant design (Dense, Factorized fixed gate, Factorized token gate, Token+Global) was trained for 33 epochs with identical optimizer and schedule.

\begin{table}[h]
\centering
\caption{Audio results at 1\% scale (18K AudioSet clips)}
\label{tab:audio_18k}
\resizebox{\linewidth}{!}{
\begin{tabular}{lcccc}
\toprule
Model & AudioSet FT mAP & AudioSet LP mAP & ESC-50 LP Acc & SC LP Acc \\
\midrule
Dense & $6.76{\pm}0.07$ & $1.18{\pm}0.01$ & $20.50{\pm}0.43$ & $11.23{\pm}0.76$ \\
Factorized (fixed gate) & $10.36{\pm}0.12$ & $1.21{\pm}0.01$ & $21.33{\pm}2.75$ & $11.70{\pm}0.46$ \\
Factorized (token gate) & $\mathbf{10.48}{\pm}0.33$ & $1.32{\pm}0.01$ & $20.67{\pm}0.63$ & $12.26{\pm}0.03$ \\
Token + Global & $10.09{\pm}0.14$ & $\mathbf{1.40}{\pm}0.01$ & $\mathbf{21.50}{\pm}0.20$ & $\mathbf{12.64}{\pm}0.10$ \\
\bottomrule
\end{tabular}
}
\end{table}

All factorized variants improved AudioSet FT mAP by $+$49--55\% relative over the dense baseline even at this small scale, demonstrating that the factorized advantage is present from the smallest dataset scale tested.

\subsection{200K-scale results}

\begin{table}[h]
\centering
\caption{Audio results at 200K scale ($\sim$10\% of AudioSet-2M)}
\label{tab:audio_200k}
\resizebox{\linewidth}{!}{
\begin{tabular}{lcccc}
\toprule
Model & AudioSet FT mAP & AudioSet LP mAP & ESC-50 LP Acc & SC LP Acc \\
\midrule
Dense & $11.04{\pm}0.25$ & $3.61{\pm}0.00$ & $40.33{\pm}0.51$ & $17.71{\pm}0.12$ \\
Factorized (fixed gate) & $13.52{\pm}0.15$ & $3.76{\pm}0.03$ & $43.42{\pm}0.94$ & $19.88{\pm}0.08$ \\
Factorized (token gate) & $13.75{\pm}0.20$ & $4.09{\pm}0.04$ & $44.42{\pm}0.42$ & $\mathbf{21.08}{\pm}0.10$ \\
Token + Global & $\mathbf{13.99}{\pm}0.15$ & $\mathbf{4.40}{\pm}0.04$ & $\mathbf{44.75}{\pm}1.27$ & $20.00{\pm}0.18$ \\
\bottomrule
\end{tabular}
}
\end{table}

At 200K, the factorized advantage persists across all four metrics. Token+Global wins 3 of 4 metrics but underperforms the token-gate model on SpeechCommands LP (20.00 vs.\ 21.08). The SpeechCommands ordering varies across scale: Token+Global is ahead at 18K (Table~\ref{tab:audio_18k}: 12.64 vs.\ 12.26), behind at 200K, and approximately tied with the token-gate model at 2M (Table~\ref{tab:audio}: 21.87 vs.\ 21.96). We therefore avoid drawing a stable architectural conclusion from this single metric and focus on the consistent factorized-vs-dense improvement.

\section{Cross-Modal Gate Mechanism Analysis}
\label{app:audio_gate}

To test whether the gate intervention findings from EEG (\S\ref{app:gate_interventions}) are EEG-specific or reflect a general property of axis-factorized attention, we run the same seven gate overrides on the audio AXON-TokenGated model (200K AudioSet pretraining). The encoder is frozen; only the linear probe head is trained. We evaluate on ESC-50 and SpeechCommands~v2.

\begin{table}[h]
\centering
\caption{Audio gate intervention results. The same seven overrides from the EEG analysis are applied to the frozen audio AXON-TokenGated encoder. $\Delta$ is the absolute change in accuracy relative to the learned baseline.}
\label{tab:audio_gate}
\small
\begin{tabular}{lcccc}
\toprule
Intervention & ESC-50 Acc & SC Acc & ESC-50 $\Delta$ & SC $\Delta$ \\
\midrule
Learned (baseline) & 44.50 & 20.93 & --- & --- \\
\midrule
Shuffled & 44.50 & 20.98 & $+$0.00 & $+$0.05 \\
Layer-mean & 44.25 & 20.71 & $-$0.25 & $-$0.22 \\
Uniform $(1/2, 1/2)$ & 44.00 & 21.01 & $-$0.50 & $+$0.08 \\
Position-mean & 43.50 & 20.59 & $-$1.00 & $-$0.34 \\
Hard argmax & 39.00 & 20.39 & $-$5.50 & $-$0.54 \\
Frequency-only ($\alpha\!=\!0$) & 29.75 & 18.03 & $-$14.75 & $-$2.90 \\
Temporal-only ($\beta\!=\!0$) & 23.75 & 13.64 & $-$20.75 & $-$7.29 \\
\bottomrule
\end{tabular}
\end{table}

\begin{figure}[h]
\centering
\includegraphics[width=0.75\linewidth]{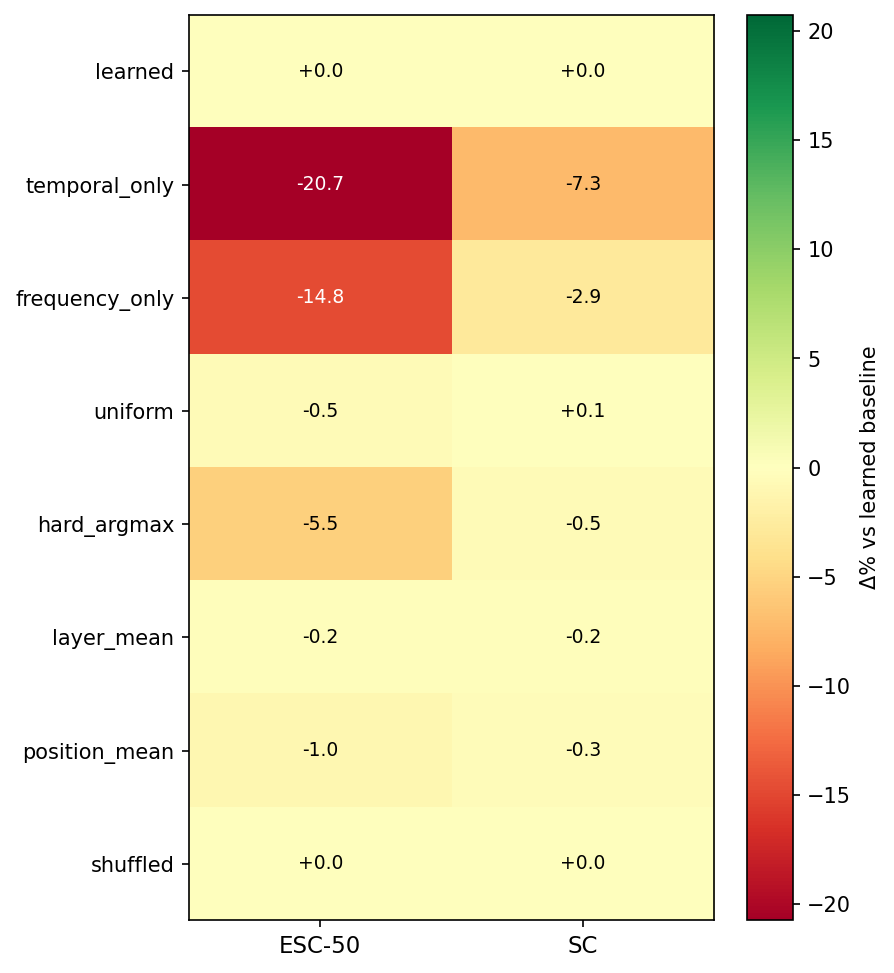}
\caption{Audio gate intervention heatmap ($\Delta$\% vs.\ learned baseline). The pattern mirrors EEG: single-axis routing is catastrophic, while shuffled and layer-mean gates are indistinguishable from the learned gate.}
\label{fig:audio_gate_heatmap}
\end{figure}

\paragraph{Cross-modal comparison.}
Table~\ref{tab:audio_vs_eeg_gate} compares the mean relative degradation of each intervention across EEG (6~tasks, balanced accuracy) and audio (2~tasks, accuracy). Both columns report the mean relative change $(\Delta / \text{baseline}) \times 100$. The interventions show a similar broad pattern across domainss: removing either axis or replacing soft mixing with hard selection is more damaging than averaging or shuffling the gates. The exact ordering and magnitudes differ: spatial-only is the worst override on EEG but temporal-only is the worst on audio, and uniform mixing costs 5.6\% on EEG but almost nothing on audio.

\begin{table}[h]
\centering
\caption{Gate intervention mean relative degradation (\%): EEG vs.\ audio. Both columns report the mean relative change from the learned baseline. The broad pattern is shared across domains; the exact ordering and magnitudes differ.}
\label{tab:audio_vs_eeg_gate}
\small
\begin{tabular}{lcc}
\toprule
Intervention & EEG Mean $\Delta$\% & Audio Mean $\Delta$\% \\
\midrule
Shuffled & $-$1.3 & $+$0.1 \\
Layer-mean & $-$2.1 & $-$0.8 \\
Position-mean & $-$0.8 & $-$1.9 \\
Uniform & $-$5.6 & $-$0.4 \\
Hard argmax & $-$12.2 & $-$7.5 \\
Spatial/Freq-only & $-$22.0 & $-$23.5 \\
Temporal-only & $-$17.5 & $-$40.7 \\
\bottomrule
\end{tabular}
\end{table}

\paragraph{Audio layer-mean gate profile.}
The calibrated per-layer gate means reveal an interpretable depth schedule. Early layers favour the frequency axis ($\alpha_1\!=\!0.422$, frequency-heavy), mid layers are approximately balanced ($\alpha_{3\text{--}5}\approx 0.485$), and late layers shift toward the temporal axis ($\alpha_{10}\!=\!0.564$, $\alpha_{11}\!=\!0.556$). This is the opposite direction from EEG, where early layers are temporal-heavy ($\alpha_0\!=\!0.745$) and mid layers are spatial-heavy. The reversal is consistent with domain structure: early audio layers capture spectral features (pitch, harmonics) that require cross-frequency integration, while late layers capture temporal dynamics (onsets, rhythm) that require cross-time integration. In EEG, the early temporal bias captures fast transient features (spikes, ERD onset) before spatial mixing integrates across electrodes.

\paragraph{Uniform robustness gap.}
The most notable cross-modal difference is the uniform intervention: $-$5.6\% in EEG but only $-$0.4\% in audio. This indicates that the audio schedule is flatter, the per-layer gate values range from $\alpha\!=\!0.422$ to $0.564$ (range~$0.14$), compared to $\alpha\!=\!0.294$ to $0.745$ (range~$0.45$) in EEG. Audio representations benefit nearly equally from both axes at all depths, while EEG requires stronger layer-varying axis preferences. This is consistent with audio spectrograms containing genuine cross-axis harmonic structure at all levels, whereas EEG temporal and spatial dynamics are more separable.

\paragraph{Summary.}
The gate interventions show the same broad pattern in both domains: both axes and soft mixing are necessary, while the dominant useful gate structure is the learned per-layer temporal/spatial balance. Token-level content routing is secondary rather than dominant in both EEG ($-$1.3\% shuffled) and audio ($+$0.1\% shuffled). The token gate serves as a training mechanism that discovers an appropriate layer-wise axis schedule, and the optimal schedule direction differs between domains (temporal-first in EEG, frequency-first in audio).

\subsection{Limitations}
\label{app:limitations}

Several limitations should be noted.
First, the theoretical support for removing the global path rests on empirical ablations (Table~\ref{tab:ablation}) rather than a formal information-theoretic proof for nonlinear masked autoencoders; such a treatment remains open.
Second, the audio experiment uses a reduced spectrogram resolution (8 frequency bins vs.\ AudioMAE's 128), limiting the absolute performance achievable; while the within-setup comparison is valid, the factorized advantage under full spectral resolution has not been verified.
Third, the pretraining corpus pools clinical EEG recordings from a limited number of sources; performance on substantially different populations or recording protocols has not been evaluated.

\subsection{Future Work}

Priority directions include: (1) repeating the AudioMAE experiment at full spectrogram resolution (128 frequency bins) to determine whether the factorized advantage persists when spectral detail is not bottlenecked; (2) investigating whether a task-conditioned gate temperature could resolve the temporal locality tradeoff across clinical and BCI tasks simultaneously; and (3) evaluating AXON on additional downstream domains such as sleep staging with polysomnography and intracranial EEG.

\end{document}